%% file: acl_latex.tex
\pdfoutput=1

\documentclass[11pt]{article}

\usepackage[]{acl}

\usepackage{times}
\usepackage{latexsym}

\usepackage[T1]{fontenc}

\usepackage[utf8]{inputenc}

\usepackage{microtype}
\usepackage{multirow}
\usepackage{booktabs}
\usepackage{times}
\usepackage{latexsym}
\usepackage{amsfonts}
\usepackage{graphicx} 
\usepackage{booktabs}
\usepackage{tabularx}
\usepackage{makecell}
\usepackage{array}
\usepackage{amsmath}
\usepackage{comment}
\usepackage{xcolor} 
\usepackage{verbatim} 
\usepackage{mdframed}
\usepackage{tabularx}
\usepackage{lipsum}
\usepackage{graphicx}
\usepackage{amsthm}
\usepackage{amssymb}
\usepackage{graphicx}
\usepackage{subfig}
\usepackage{thmtools}
\usepackage{thm-restate}

\theoremstyle{definition}
\newtheorem{definition}{Definition}[section]

\newcommand{\kbset}[1]{edit boundary}
\newcommand{\kbsets}[1]{knowledge boundaries}
\newcommand{\benchmark}[1]{\textsc{EvEdit}}

\title{\benchmark{}: Event-based Knowledge Editing with Deductive Editing Boundaries
}

\author{Jiateng Liu$^{\heartsuit}$\thanks{~~Equal contribution}~~~~Pengfei Yu$^{\heartsuit}$\footnotemark[1]~~~~Yuji Zhang$^{\clubsuit}$\thanks{~~Work is done during the internship at UIUC}~~~~Sha Li$^{\heartsuit}$~~~~Zixuan Zhang$^{\heartsuit}$~~~~Heng Ji$^{\heartsuit}$
\vspace{2pt}
\\
\vspace{2pt}
$^{\heartsuit}$University of Illinois Urbana-Champaign~~$^{\clubsuit}$The Hong Kong Polytechnic University
        \\ \vspace{2pt}
        \texttt{\{jiateng5, pengfei4, yujiz, shal2, zixuan2, hengji\}@illinois.edu}
        }
        
\begin{document}
\maketitle

\input{Contents/abstract}

\input{Contents/01_intro}
\input{Contents/02_formulation}

\input{Contents/03_methods}
\input{Contents/04_evaluation}

\input{Contents/05_related_work}

\input{Contents/06_Conclusion}
\input{Contents/limitation}

\bibliography{anthology,custom}
\bibliographystyle{acl_natbib}

\input{Contents/appendix}

\end{document}

%% file: Contents/abstract.tex
\begin{abstract}

The dynamic nature of real-world information necessitates efficient knowledge editing (KE) in large language models (LLMs) for knowledge updating. However, current KE approaches, which typically operate on (subject, relation, object) triples, ignore the contextual information and the relation among different knowledge. Such editing methods could thus encounter an uncertain editing boundary, leaving a lot of relevant knowledge in ambiguity: Queries that could be answered pre-edit cannot be reliably answered afterward. In this work, we analyze this issue by introducing a theoretical framework for KE that highlights an overlooked set of knowledge that remains unchanged and aids in knowledge deduction during editing, which we name as the \emph{deduction anchor}.  We further address this issue by proposing a novel task of \emph{event-based knowledge editing} that pairs facts with event descriptions. This task manifests not only a closer simulation of real-world editing scenarios but also a more logically sound setting, implicitly defining the deduction anchor to address the issue of indeterminate editing boundaries. We empirically demonstrate the superiority of event-based editing over the existing setting on resolving uncertainty in edited models, and curate a new benchmark dataset \benchmark{} derived from the {\textsc{CounterFact}} dataset. Moreover, while we observe that the event-based setting is significantly challenging for existing approaches, we propose a novel approach \emph{Self-Edit} that showcases stronger performance, achieving 55.6\% consistency improvement while maintaining the naturalness of generation.
\footnote{Data and code released at: https://github.com/Lumos-Jiateng/EvEdit}
\end{abstract}

%% file: Contents/01_intro.tex
\section{Introduction}
\input{Inserts/1setting}

The dynamics of the physical world underscore the importance of knowledge editing~(KE) for large language models~\cite{yao2023editing,wang2023knowledge,zhang2024comprehensive}. This line of research aims at updating models' beliefs and shaping models' behaviors based on the editing knowledge for improved accuracy and usability. However, real-world edits usually originate from emerging events that encompass logical connections between new facts and past facts. Therefore, it is insufficient to update a single fact, as the change should propagate through other related facts that can be inferred from the logical connections. For instance, by updating the model with ``Messi joined team Inter Miami", the edited model should acknowledge that ``Messi began playing in Major League Soccer (MLS)", as "Inter Miami" belongs to MLS. This is referred to as the \textit{ripple effect} in \citet{cohen2023evaluating}. While they identify cases where such ripple effects are present, there are many other cases, as shown in Figure \ref{fig:dilemma}, where the ripple effect is uncertain. 

The increased uncertainty of the model after editing results from the current KE setting which defines an edit as a single \textit{(subject, relation, object)} triple without any related context. 
This leaves a substantial portion of relevant knowledge in a gray area of indecision. 
To analyze this phenomenon, we present a more formal definition of knowledge editing based on the \emph{formal logic}~\cite{smith2003introduction}, representing knowledge as formal language propositions. We define two concepts playing significant roles in maintaining logical soundness of the editing process: \textbf{\emph{Deduction anchor}} is a set of pre-edit knowledge preserved through the editing process, as exemplified in Figure~\ref{fig:dilemma}; 
\textbf{\emph{Editing boundary}} is the comprehensive set of desired knowledge changes inside edited models inferred from the union of the deduction anchor and the editing knowledge.

We find that existing work overlooks the deduction anchor while implicitly holding flawed assumptions: either the \emph{no-anchor assumption} (an empty anchor set) or the \emph{max-anchor assumption} (an anchor set comprising all knowledge not conflicting with the edit). We demonstrate that these improper assignments of the deduction anchor lead to a problem of knowledge explosion for counterfactual edits where edited models need to accept conflicting facts for logical soundness. Consequently, the existing setting theoretically increases the uncertainty over related knowledge,  manifesting as the shrinkage of knowledge within edited language models.

To overcome the limitations of the current setting, we introduce \textit{event-based knowledge editing}, which proves not only a more robust setting by presenting clearer deduction anchors and editing boundaries, but also a more practical setting in the real world, because edits are often organized in events~\cite{chen2021joint,chen2021event}.
We derive a new benchmark \benchmark{} from a triple-based knowledge editing benchmark \textsc{counterfact}~\cite{meng2022locating} by augmenting facts with events using GPT-3.5-turbo. 
By measuring the certainty over knowledge after edits as in-context answer probabilities, we verify empirically 
that while the existing setting inherently \textbf{\textit{requires}} a substantial reduction of the knowledge within edited models, the event-based setting significantly mitigates the issue.

For editing approaches, we decompose the event descriptions into a series of triples for subsequent accommodations of current editing methods like Rome~\cite{meng2022locating}, MEMIT~\cite{meng2022mass}, PMET~\cite{li2023pmet} and Grace~\cite{hartvigsen2023aging}. We further propose a novel solution Self-edit inspired by \citet{yu2023self} which can effectively utilize the eventual context to decide editing boundaries during updating. With the \benchmark{} benchmark, we systematically evaluate the performance of both methods on both a text-completion task and a question-answering task.
Our evaluations show that while adapting previous editing approaches provides suboptimal results, our approach exhibits over 56.6\% increase in factual consistency while keeping the naturalness of generations by edited models.


Overall, our contributions are:
\begin{enumerate}
    \item We identify a critical deficiency of the current KE setting, by providing a careful theoretical analysis, we attribute the problem to the improper assignment of deduction anchor.
    \item We propose event-based knowledge editing and a new benchmark \benchmark{}, addressing the problem of improper anchors and aligning well with real-world scenarios. We empirically validate the deficiency of the current settings and the superiority of our setting.
    \item We propose a novel Self-edit approach for the new setting, significantly outperforming existing methods on consistency and naturalness.
\end{enumerate}


%% file: Inserts/1setting.tex
\begin{figure*}[t]
    \centering
    \includegraphics[width=1.0\textwidth]{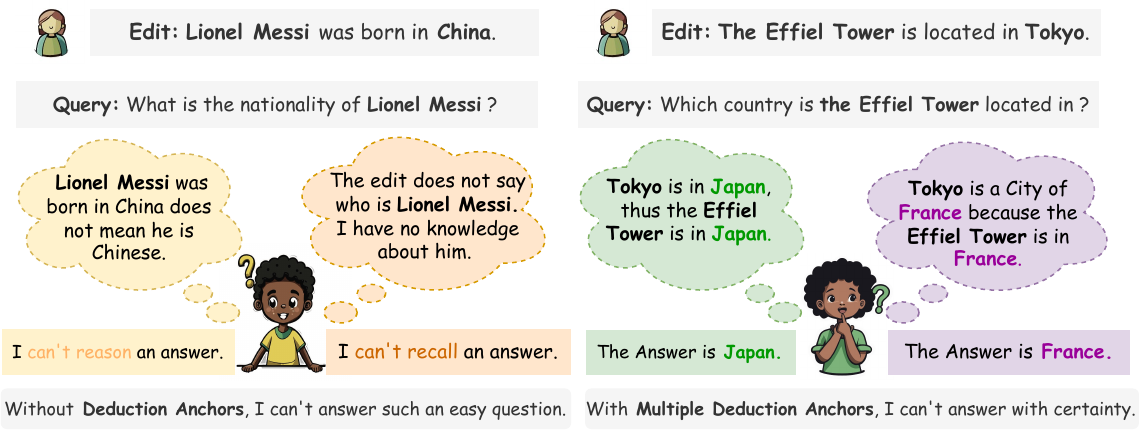}
    \vspace{-17pt}
    \caption{We observe fallacies of existing knowledge editing when the Deduction Anchor is not defined. On the right, \textit{Tokyo is in Japan} and \textit{Effiel Tower is in France} are candidate elements for the Deduction Anchor. \textbf{Left:}The No-Anchor Fallacy in Theorem~\ref{Thereom:no-anchor},  where defining an empty deduction anchor without additional knowledge hinders effective reasoning. 
    \textbf{Right:} The Max-anchor Fallacy in Theorem~\ref{Thereom:max-anchor}, where defining the deduction anchor as the entire model knowledge fails due to uncertainty from alternative reasoning chains. 
    } 
    \label{fig:dilemma}
\end{figure*}

%% file: Contents/02_formulation.tex
\section{Fallacies of Knowledge Editing}
\label{formulation}
In this section, we formulate and analyze the task of knowledge editing from both theoretical and empirical perspectives. In \S~\ref{sec:knowledge}, we present a theoretical formulation for knowledge editing. In \S~\ref{sec:deficiency}, we theoretically analyze the fallacies of the existing work and empirically validate its failure.

\subsection{Formulation of Knowledge Editing}
\label{sec:knowledge}
We present a theoretical framework for knowledge editing based on the \emph{formal logic}~\cite{smith2003introduction} where we consider \emph{knowledge} as \emph{propositions}\footnote{Propositions are arguments that can be either true or false}. For a knowledge system, the purpose of knowledge editing is to alter its set of knowledge. Therefore, we first formally define the knowledge within a system and the knowledge edit.

\begin{definition}[Knowledge of Systems]
    The \emph{knowledge} of a system is a set of propositions that is considered true in the system.
\end{definition}

To align the theoretical framework with language model~(LM) editing, we introduce the knowledge of LMs. Let $k$ denote a proposition, and let $\Theta$ represent an LM. We assess whether $\Theta$ "possesses" knowledge of $k$ by calculating $P(y_k|x_k, \Theta)$, where $(x_k, y_k)$ represents a pair of input-output tokens to verify the knowledge. For example, we may use $x_k=$ `Messi was born in' and $y_k=$ `Argentina' to examine the knowledge of the birthplace of Messi. We opt for $P(y_k|x_k, \Theta)$ over $P(k|\Theta)$ because the probability assigned by a language model to a proposition does not inherently correlate with its logical validity~\cite{yu2023self}
\begin{definition}[Knowledge of LMs]\label{def:lm_knowledge}
For a language model $\Theta$, the universe of all conceivable knowledge $\mathcal{U}$, and a threshold $\varepsilon$ within the range $[0, 0.5)$, the set of knowledge recognized by $\Theta$ is
\begin{equation}\label{eq:def_lm_knowledge}
\mathcal{K}_{\Theta,\varepsilon} = \{k\in\mathcal{U}|P(y_k|x_k, \Theta) \ge 1 - \epsilon\}.
\end{equation}
There could be multiple candidate $\{(x_k^i, y_k^i)\}$ verifying the same knowledge $k$. We can replace $P(y_k|x_k, \Theta)$ with a random sample, mean or maximum of all candidates' probabilities in Equation~(\ref{eq:def_lm_knowledge}) with no influence on the rest of the formulation. Therefore, we simply use $P(y_k|x_k, \Theta)$ for brevity.
\end{definition}
In this work, we are specifically concerned with the logical deduction during editing such as:

\begin{centering}
    $\boldsymbol{P}:$ Tom was born in the city of New York\\
    $\boldsymbol{Q}:$ The country of New York is U.S.\\
    $\downarrow$\\
    $\boldsymbol{X}:$ Tom was born in the country of U.S.\\
\end{centering}
\noindent For a knowledge set $\mathcal{K}$, its \emph{deductive closure} $\mathcal{B}(\mathcal{K})$ is the set of all propositions logically entailed by $\mathcal{K}$. $\mathcal{K}$ is \emph{deductively closed}, or simply \emph{closed}, if and only if $\mathcal{B}(\mathcal{K}) = \mathcal{K}$. Determining the deductive closure presents a significant challenge due to the difficulty in formulating deduction rules~\cite{smith2003introduction}. However, given the advanced in-context reasoning capabilities demonstrated by large language models, we establish the deductive closure based on such in-context deduction.
\begin{definition}[In-context Deductive Closure]\label{def:ict_dc}
For any given set of knowledge $\mathcal{K}$, its \emph{In-context Deductive Closure} as provided by a language model $\Phi$ is the set of knowledge that can be deduced,
\begin{equation}\label{eq:def_ictdc}
\mathcal{B}_{\Phi,\epsilon}(\mathcal{K})=\{u\in\mathcal{U}|P(y_u|x_u, \mathcal{K}, \Phi) \ge 1 - \epsilon\}. 
\end{equation}
\end{definition}
We assume a language model being edited to contain closed set of knowledge for the mathematical soundness of the formulation, although it may be only approximately true in current LLMs. Let $\mathcal{K}$ be the knowledge set of the pre-edit model, and $\mathcal{E}$ be the set of editing knowledge. We define two novel concepts for the soundness of editing: \emph{deduction anchor} and \emph{editing boundary}.
\begin{definition}[Deduction Anchor of Editing] The \emph{deduction anchor} of an edit is a subset of the current knowledge assumed true throughout editing.
\end{definition}
We denote the deduction anchor by $\mathcal{K}^\mathcal{E}$, which serves as the base for generalization of the editing knowledge. We now define the editing boundary.
\begin{definition}[Editing Boundary]
    The editing boundary is the closed set $\mathcal{B}\left(\mathcal{K}^\mathcal{E}\cup \mathcal{E}\right)$ of logically relevant knowledge to the edit $\mathcal{E}$.
\end{definition}
We thereby define \emph{knowledge editing}.
\begin{definition}[Knowledge Editing] Given the knowledge set $\mathcal{K}$, the edit $\mathcal{E}$ and the deduction anchor $\mathcal{K}^\mathcal{E}$, \emph{knowledge editing} is the process of computing edited knowledge set $\mathcal{K}'$:
    \begin{equation}\label{eq:def_ke}
    \begin{aligned}
        \mathcal{K}^D &= \left\{p\in\mathcal{K}| \neg p\in  \mathcal{B}\left(\mathcal{K}^\mathcal{E}\cup \mathcal{E}\right) \right\}\\
        \mathcal{K}' &= \mathcal{B}\left(\mathcal{K}\backslash\mathcal{K}^D \cup \mathcal{E}\right)
    \end{aligned},
\end{equation}
where $\mathcal{K}^\mathcal{E}$ satisfies that
\begin{equation}\label{eq:req_da}
    \forall k \in \mathcal{B}\left(\mathcal{K}\backslash\mathcal{K}^D\right), \neg k \notin \mathcal{B}\left(\mathcal{K}^\mathcal{E}\cup \mathcal{E}\right) .
\end{equation}
Here Equation~(\ref{eq:req_da}) ensures the consistency of $\mathcal{K}'$. $\mathcal{K}^D$ is the set of knowledge conflicting with the deducted knowledge from $\mathcal{K}^\mathcal{E}\cup \mathcal{E}$, which needs to be erased from the model being edited.
\end{definition}
We also define knowledge editing of language models. It's important to note that the model used to determine the deductive closure in Equation~(\ref{eq:def_ictdc}) serves only in defining the task and not in the editing process itself. Thus, it may differ from the model undergoing edits. We may employ a more capable model for a better deductive closure. 
\begin{definition}[Knowledge Editing of LMs]\label{eq:def_kelm}
    Following the notations in Equation~(\ref{eq:def_ke}), to edit a language model $\Theta$ based on the in-context deductive closure provided by $\Phi$ involves identifying a modified model $\Theta'$ such that
    \begin{equation}
    \begin{aligned}
        \mathcal{K}^D &= \left\{p\in\mathcal{K}_{\Theta,\varepsilon_\Theta}| \neg p\in  \mathcal{B}_{\Phi, \varepsilon_\Phi}\left(\mathcal{K}^\mathcal{E}\cup \mathcal{E}\right) \right\}\\
        \mathcal{K}' &= \mathcal{B}_{\Phi, \varepsilon_\Phi}\left(\mathcal{K}_{\Theta,\varepsilon_\Theta}\backslash\mathcal{K}^D \cup \mathcal{E}\right)
    \end{aligned}.
    \end{equation}
where $\mathcal{K}_{\Theta,\varepsilon_\Theta}$ and $\mathcal{B}_{\Phi, \varepsilon_\Phi}$ are defined in Definition~\ref{def:lm_knowledge} and Definition~\ref{def:ict_dc}, respectively.
\end{definition}

\input{Inserts/certainy}
\subsection{Fallacies of Existing Knowledge Editing}
\label{sec:deficiency}

Existing work predominantly ignores the significance of the deduction anchor and resulting editing boundary without explicit characterizations of them. 
They mostly focus on local edits assuming $\mathcal{K}^\mathcal{E} = \varnothing$, which limits the editing boundary 
$\mathcal{B}\left(\mathcal{E}\right)$ to only contain paraphrases of $\mathcal{E}$, as the \emph{edit scope} proposed by \citet{mitchell2022memory}. Additionally, \citet{cohen2023evaluating} implicitly assumes all knowledge not directly conflicting with $\mathcal{E}$ as the deduction anchor.  However, we present the following theorems, emphasizing the importance of choosing an appropriate set of $\mathcal{K^\mathcal{E}}$ and summarizing fallacies under their flawed assumptions.
\begin{restatable}[Knowledge Explosion]{theorem}{thmkexp}
\label{thm:kexp}
    If Equation~(\ref{eq:req_da}) is not satisfied, the edited knowledge set $\mathcal{K}'=\mathcal{U}$ where $\mathcal{U}$ is the universe of all knowledge, meaning any proposition is logically true.
\end{restatable}

\begin{restatable}[No-Anchor Fallacy]{theorem}{thmno}\label{Thereom:no-anchor} For a counterfactual and non-local edit $\mathcal{E}$, there exists $\mathcal{K}^\mathcal{E}\in2^\mathcal{K}$ satisfying Equation~(\ref{eq:req_da}), while $\varnothing$ does not.
\end{restatable}
\begin{restatable}[Max-Anchor Fallacy]{theorem}{thmmax}\label{Thereom:max-anchor} For a counterfactual and non-local edit $\mathcal{E}$, 
the max-anchor $\{p\in\mathcal{K}| \neg p \notin \mathcal{B}(\mathcal{E})\}$
does not satisfy Equation~(\ref{eq:req_da}).
\end{restatable}
Here a \emph{counterfactual} and \emph{non-local} edit is one that contradicts with some but not all of the pre-edit knowledge. The rigorous definitions are presented with proofs of the above theorems in Appendix~\ref{app:proof}.

Moreover, the knowledge explosion leads to the shrinkage of the knowledge set of language models following Equation~(\ref{eq:def_lm_knowledge}). The reason is that for two conflicting knowledge elements $p, q$ where $x_p=x_q, y_p\ne y_q$, a language model cannot assign $P(y_p|x_p) \ge 1 - \varepsilon$ and $P(y_q|x_q) \ge 1 - \varepsilon$ at the same time. Consequently, we hypothesize that both probabilities will go under the threshold of $1 - \varepsilon$, causing uncertainty within models. We further verify the hypothesis empirically with a set of paired edits and relevant knowledge queries as follows:
\begin{center}
    
    \textbf{Edit} $e$: \emph{A} is located near to \emph{B}. 
    
    \textbf{Query} $q$: \emph{A} is located in the country of \_
\end{center}

\noindent where \emph{A} and \emph{B} are two cities in different countries. For each $q$, we compare the pre-edit certainty $\max_y P(y|q, \Theta)$ with the edited certainty $\max_y P(y|q, e, \Theta)$ for various $e$ with different choices of \emph{B} in Figure~\ref{fig:uncertainty}, which demonstrates the predicted decrease in certainty. Moreover, the magnitude of the decrease appears to be larger for models with stronger reasoning abilities.




%% file: Inserts/certainy.tex
\begin{figure*}[!ht]
    \centering
    \includegraphics[width=0.98\textwidth]{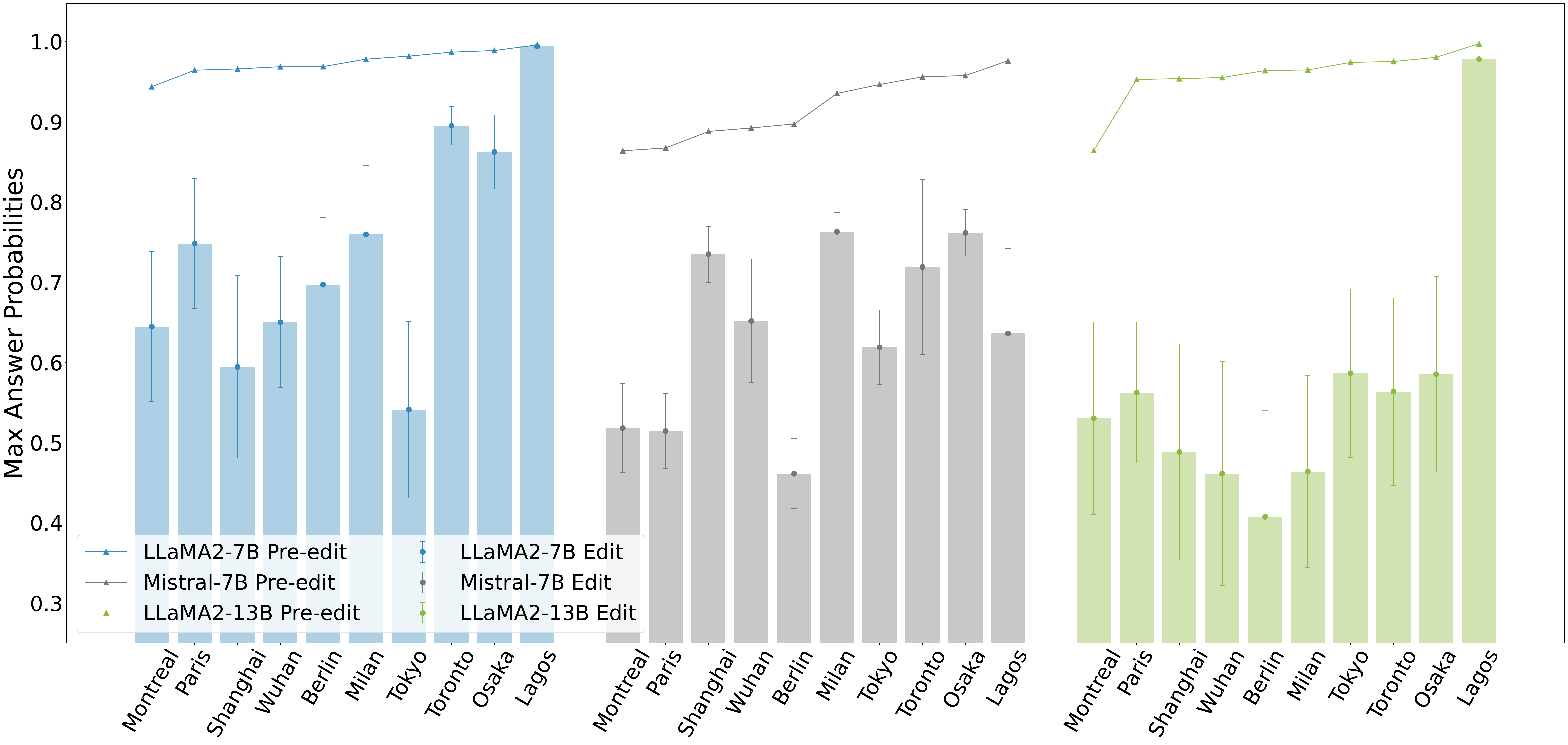}
    \vspace{-5pt}\caption{Counterfactual edits reduces model's certainty on relevant knowledge. We measure the uncertainty as the maximum answer probability to the query of ``\emph{A} is located in the country of \_'' where \emph{A} is one of the cities labeled in the X-axis. We compute the range of ``Edit'' probabilities by prepending various counterfactual edits as context to the query. ``Pre-edit'' probabilities are model predictions without any context. See main text for more details.} 
    \label{fig:uncertainty}
    
\end{figure*}

%% file: Contents/03_methods.tex
\section{Event-Based Knowledge Editing}
\input{Inserts/Event}

Following our prior analysis, edits without sufficient context to infer a proper deduction anchor $\mathcal{K}^\mathcal{E}$ cause conflicts between the editing boundary $\mathcal{B}\left(\mathcal{K}^\mathcal{E} \cup \mathcal{E}\right)$ and the remaining model knowledge $\mathcal{K}\backslash\mathcal{K}^D$, which ultimately lead to uncertainty in edited models. Rather than simply augmenting the existing benchmarks with deduction anchors for edits, we propose a more practical setting of augmenting edits with eventual context since knowledge updates are more often driven by events in real-world scenarios~\cite{chen2021event,chen2021joint} rather than provided deduction anchors. As shown in Figure~\ref{fig:event} each edit is an event description under event-based editing, where event descriptions helps define the deduction anchor. This setting improves both logically robustness and real-world applicability. 

\subsection{The \benchmark{} Benchmark}
We compile our event-based knowledge editing benchmark \benchmark{} from the \textsc{CounterFactual} dataset~\cite{meng2022locating}, where each instance is a single fact to update. The procedure as described below can also be applied to other knowledge-editing datasets. 
Data statistics and examples of are detailed in Appendix~\ref{appendix:dataset} and the prompts for data creation are in Appendix~\ref{appendix:prompts}. 

\paragraph{Data Collection}
We begin with using GPT-3.5-turbo~(referred as GPT later) to filter out edits that are impossible to take place as future events. 
We then prompt GPT with in-context examples to generate an event description for each remaining edit. This step is essentially using GPT to define an implicit deduction anchor, as exemplified in Figure~\ref{fig:event}. 


\paragraph{Evaluation Task}
To systematically evaluate the abilities of edited models, we include both the question-answering task and the text-completion task. For each edit, we generate five related question-answer~(QA) pairs using GPT. We also require one question to be undecidable given the event description to better delineate the editing boundary by considering GPT as $\Phi$ in Definition~\ref{def:ict_dc}, for which we provide the answer as ``I don't know''~\cite{zhang2023rtuning}. We split the evaluation set into the ``Known'' set and the ``Unknown'' set accordingly. These QAs are subsequently transformed into text completion tasks.

\input{Inserts/3}
\subsection{Event-based Edits Mitigates Uncertainty}
\label{experiments:setting}
We verify that event-based editing reduces uncertainty compared with single factual edits.
using event descriptions as a replacement for the fact achieves our goal of mitigating uncertainty. 
We quantify uncertainty based on Equation~(\ref{eq:def_lm_knowledge}). However, since it is computational costly to compute $\max_y P(y|x, \Theta)$ for longer output sequences (answers or text completions), we instead use $\mathbb{E}_{y\sim P(y|x, \Theta)}\log P(y|x, \Theta)$ to measure the certainty\footnote{We sample 5 answers and average the log-likelihood.}.
Each edit instance in \emph{E$^2$dit} contains the original fact, the event description, and the question-answer pairs related to the fact.
We compare the certainty of a frozen pretrained LM generating answers to questions when given the original fact versus the event description.
We plot our results on LLaMA2-7B-Chat in Figure~\ref{fig:3} and leave results on Mistral-7B, and LLaMA2-13B-Chat in Appendix~\ref{appendix:more results}. Each edit instance corresponds to a point in the scatter plot.
We use red to highlight instances where event-based context enhances generation certainty, and blue to indicate the opposite case. Results show that event-based knowledge editing significantly reduces uncertainty.

\input{Inserts/method}

\subsection{Approach: Self-Edit Framework}
\label{sec4:appraoch}
Inspired by~\cite{yu2023self}, we design a Self-edit approach for event-based editing.
Given the event-based edits, we use the pre-edit language model to create an augmented dataset to fine-tune the model. As on the right side of Figure~\ref{fig:method}, for each edit $\mathcal{E}$,

\begin{enumerate}
    \item Conduct self-prompting of the language model to generate a related question $Q$ 
    \item Generate the answer $A$ 
(\textit{2024}) 
by prompting the LM with the question $Q$ and the edit $\mathcal{E}$. For questions not answerable, we ask the model to respond with ``\textit{I don't know}''.
    \item Create a training instance of the format $(Q \rightarrow \mathcal{E}, A)$. The model is fine-tuned to recite the edit before answering the question.
\end{enumerate}
More examples are shown in Appendix~\ref{appendix:method_data_example}. Compared with \cite{yu2023self}, we also require the model to generate ``I don't know'' for better identification of its own knowledge boundary. We fine-tune the model on the resulting dataset. Generated edits before answers are removed for evaluation. 

%% file: Inserts/Event.tex
\begin{figure}[t]
    \centering
    \includegraphics[width=0.5\textwidth]{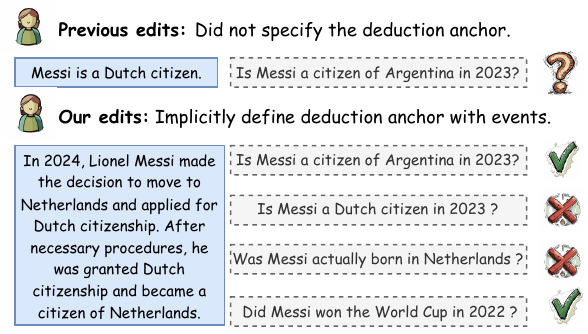}
    \vspace{-15pt}
    \caption{Event descriptions helps to define the deduction anchor of Editing implicitly.}
    \label{fig:event}
\end{figure}

%% file: Inserts/3.tex
\begin{figure}[t]
    \centering
    \includegraphics[width=0.35\textwidth]{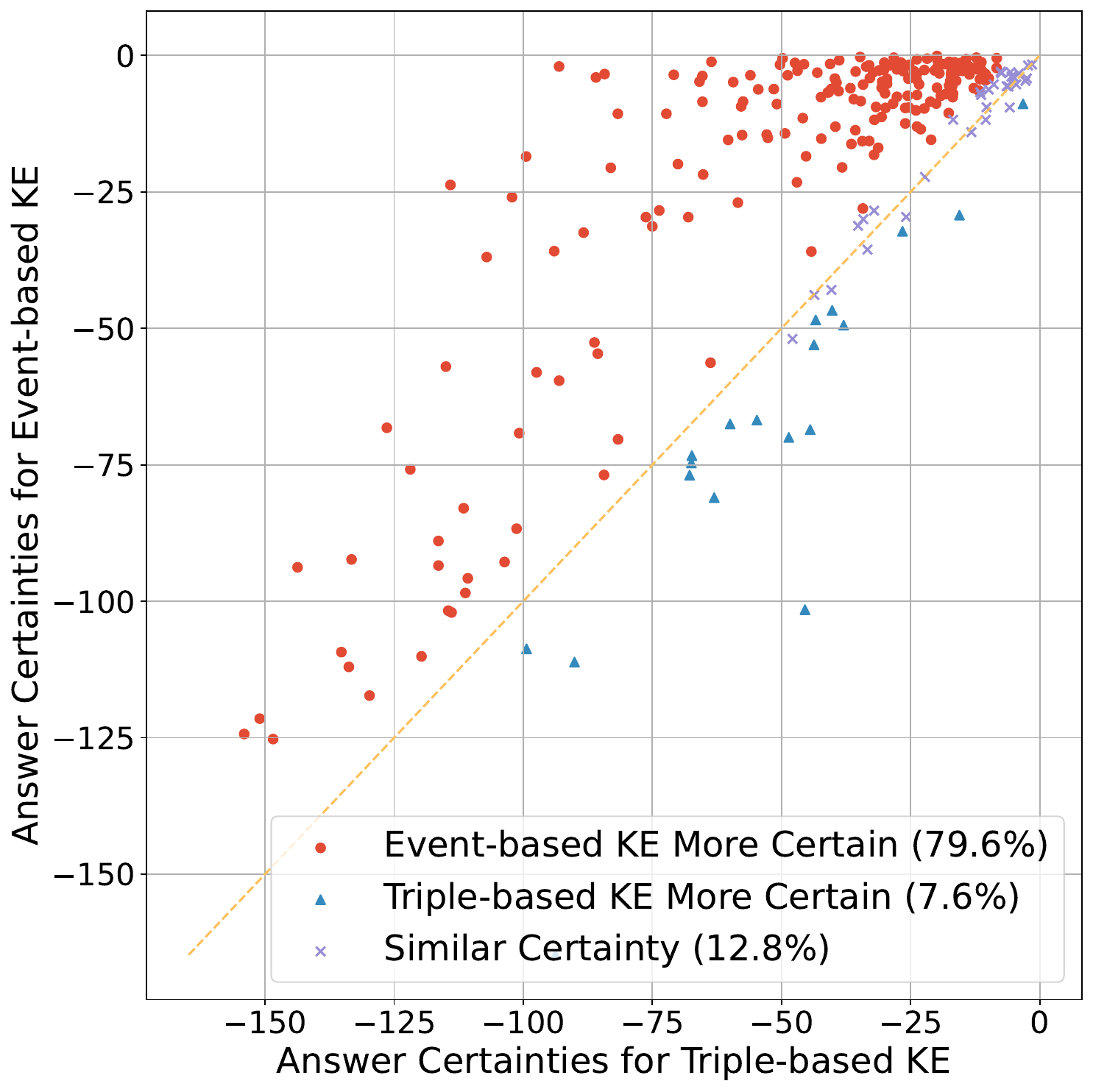}
    \vspace{-3pt}
    \caption{Event-based setting decreases uncertainty.}
    \label{fig:3}
\end{figure}

%% file: Inserts/method.tex
\begin{figure*}[t]
    \centering
    \includegraphics[width=1.0\textwidth]{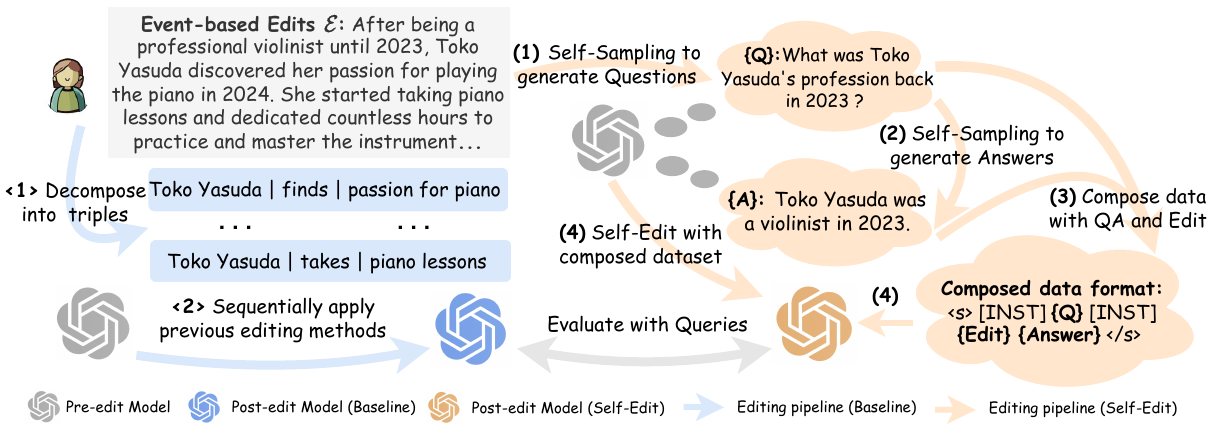}
    \vspace{-18pt}
    \caption{
     Different approaches to event-based knowledge editing. Left: To apply factual-association-based editing methods, we decompose event-based description into triples. Right: Our proposed Self-Edit: We first use the pre-edit LM to generate relevant QA pairs to edits. Then we fine-tune models on instances of (Q $\rightarrow$ Edit, A).  }
    \label{fig:method}
    \vspace{-15pt}
\end{figure*}

%% file: Contents/04_evaluation.tex
\section{Experiments}
\label{Experiments}
\subsection{Experimental Settings}
We use the LLaMA2-7B-chat model as the basis for editing and evaluate both the question-answering and the text completion setting of \textsc{EvEdit}, with the number of edits ($N$) varied to match the limitations of different baselines as specified in Section~\ref{sec:baseline}. Performance was assessed separately on ``Related" and ``Unknown" data subsets. We provide further details in Appendix \ref{appendix:experiments}.

We adopt the factual consistency metric and naturalness metric from UniEval~\cite{zhong2022unified} as our evaluation metrics. The consistency metric is computed with respect to the edit, so it measures how well the edit is executed. 
The naturalness metric shows how well the model's generation ability is preserved after the edit. 
 
\subsection{Baselines Methods in Comparison} \label{sec:baseline}

We consider three categories of baselines:

\paragraph{Factual-Association} We adapt existing factual-association editing methods to event-based editing by decomposing the event into several fact triples with GPT-3.5-turbo, as depicted on the left side of Figure \ref{fig:method}. 
 We consider \textbf{ROME}~\cite{meng2022locating}, \textbf{MEMIT}~\cite{meng2022mass}, \textbf{PMET}~\cite{li2023pmet} and \textbf{GRACE}~\cite{hartvigsen2023aging} in this category.
 These methods, however, do not scale well in terms of 
efficiency and 
effectiveness, thus we limit our evaluation to $N=1,10$.
\paragraph{Fine-tuning} For this category, we fine-tune models on $N = 100$ edits and assess their performance on $N = 1, 10, 100$ in the evaluation sets. 
We consider the \textbf{Direct Fine-tuning} (on event descriptions) and our proposed \textbf{Self-edit} in this category. 
Compared to factual-association methods, fine-tuning methods support the editing of a large number of facts simultaneously.
\paragraph{In-context Learning} Additionally, we assess an in-context performance~(ICL), which involves prepending event descriptions to evaluation prompts without changing model parameters. This serves as an upper bound based on the model's deductive capabilities\footnote{This is not a theoretical upper bound of all models' or human's logical deductions abilities, but instead an empirical upper bound only for the pre-edit model.} since it is equivalent to setting $\Phi=\Theta$ in Definition~\ref{eq:def_kelm}.
However, this method's scalability is limited by the model's context window size, thus we only evaluate it for $N=1,10,59$, where $59$ is the maximum event descriptions we can accommodate into LLaMA2.
\subsection{Main Results}
We present results for both text completion and QA tasks, across various numbers of edits $N$ and data splits (Known and Unknown) in Table \ref{table:main_results}. More qualitative results can be found in Appendix~\ref{appendix:edit_data_example}.

\paragraph{Factual Association Fails \textsc{EvEdit}}
\label{experiments:baseline}
\input{Inserts/Main_results}

Factual-association methods display limited improvements in factual consistency while significantly harming the naturalness of generations. A typical case is that tokens from the event description are generated repeatedly, as shown in Appendix~\ref{appendix:edit_data_example}.
Among this family of methods, GRACE~\cite{hartvigsen2023aging}, which employs a code book as an external repository for potential hidden states, performs best in consistency. However, GRACE is sensitive to the choice of hyperparameters, as shown by the difference in performance for different $\epsilon$ values.

\paragraph{Self-edit Excels at \textsc{EvEdit}}
\label{experiments:self-updating}
In general, fine-tuning approaches support a large number of edits with little loss in naturalness.
Compared to direct fine-tuning, our Self-Edit framework yields a substantial improvement on consistency, showing that the edit is effective. Moreover, our method displays clearer editing boundaries by improved scores on the Unknown subset. Since our method is fine-tuned with explicit editing boundaries by giving ``I don't know'' for undecidable questions, we can directly compute precision, accuracy, and F1-score for the ``Unknown'' subset in Table~\ref{table:unknown-f1}. Results demonstrate that although our approach showcases improved performance over baselines on this subset, there is still a significant gap toward a satisfying characterization of editing boundaries in edited models.

\input{Inserts/edit_unknown}

\paragraph{In-context Learning Fails on Larger $N$} In-context learning shows superior performance for smaller $N$s, but its performance drops significantly as $N$ increases and is inferior to our method for $N=59$. Moreover, this method cannot scale to larger $N$ due to limited sequence length.


%% file: Inserts/Main_results.tex
\begin{table*}[!ht]
\centering
\vspace{12pt}
\small 
\resizebox{1.0\textwidth}{!}
{
\begin{tabular}{cc|cccc|cccc}
\toprule
\multirow{3}{*}{\textbf{KE methods}} & \multirow{3}{*}{\textbf{Evaluation Metric}} & \multicolumn{4}{c|}{\textbf{Text Completion}} & \multicolumn{4}{c}{\textbf{Question Answering}} \\ 
& & \multicolumn{2}{c}{\textbf{N=1}} & \multicolumn{2}{c|}{\textbf{N=10}} & \multicolumn{2}{c}{\textbf{N=1}} & \multicolumn{2}{c}{\textbf{N=10}} \\
& & \textbf{Related} & \textbf{Unknown} &  \textbf{Related} & \textbf{Unknown} & \textbf{Related} & \textbf{Unknown} &  \textbf{Related} &  \textbf{Unknown} \\
\midrule
\multirow{2}{*}{Base Model}& 
\textit{Consistency}  & 0.324 & 0.347 & 0.318 & 0.355 & 0.347 & 0.372 & 0.349 & 0.378 \\
& \textit{Naturalness}  & 0.894 & 0.869 & 0.898 & 0.875 & 0.833 & 0.821 & 0.845 & 0.866 \\
\midrule
\multirow{2}{*}{ROME}& 
\textit{Consistency}  & 0.331 & 0.262 & 0.310 & 0.258 & 0.344 & 0.270 & 0.336 & 0.243 \\
& \textit{Naturalness}  & 0.671 & 0.479 & 0.610 & 0.454 & 0.655 & 0.440 & 0.574 & 0.451 \\
\midrule
\multirow{2}{*}{MEMIT}
& \textit{Consistency}  & 0.334 & 0.277 & 0.329 & 0.271 & 0.342 & 0.281 & 0.340 & 0.279 \\
& \textit{Naturalness}  & 0.629 & 0.466 & 0.588 & 0.430 & 0.630 & 0.464 & 0.546 & 0.421 \\
\midrule
\multirow{2}{*}{PMET}& 
\textit{Consistency}  & 0.346 & 0.319 & 0.332 & 0.317 & 0.350 & 0.316 & 0.354 & 0.322\\
& \textit{Naturalness}  & 0.840 & 0.812 & 0.880 & 0.862 & 0.814 & 0.790 & 0.822 & 0.793 \\
\midrule
\multirow{2}{*}{GRACE$_{\epsilon=25}$}& 
\textit{Consistency}  & 0.436 & 0.320 & 0.442  & 0.304 & 0.441  & 0.317 & 0.443 & 0.340 \\
& \textit{Naturalness}  & 0.702 & 0.672 & 0.691 & 0.643 & 0.690 & 0.668 & 0.673 & 0.659 \\
\midrule
\multirow{2}{*}{GRACE$_{\epsilon=50}$}& 
\textit{Consistency}  & 0.337 & 0.298 & 0.335 & 0.256  & 0.345 & 0.308 & 0.344 & 0.313 \\
& \textit{Naturalness}  & 0.806 & 0.791 & 0.760 & 0.770 & 0.797 & 0.764 & 0.758 & 0.723 \\
\midrule
\multirow{2}{*}{ICL}& 
\textit{Consistency}  & \textbf{0.726} & \underline{0.351} & \textbf{0.626} & \underline{0.331}  & \textbf{0.739} & \textbf{0.405} & \textbf{0.662} & \underline{0.350} \\
& \textit{Naturalness}  & 0.903 & 0.887 & 0.913 & 0.896 & 0.898 & 0.846 & 0.910 & 0.902 \\
\midrule
\multirow{2}{*}{\textbf{Ours}}& 
\textit{Consistency}  & \underline{0.512} & \textbf{0.401} & \underline{0.507} & \textbf{0.402} & \underline{0.523} & \underline{0.403} & \underline{0.519} & \textbf{0.388} \\
& \textit{Naturalness}  & 0.804 & 0.867 & 0.816 & 0.877 & 0.816 & 0.872& 0.817 & 0.864 \\
\midrule
\midrule
\multirow{3}{*}{\textbf{KE methods}} & \multirow{3}{*}{\textbf{Evaluation Metric}} & \multicolumn{4}{c|}{\textbf{Text Completion}} &\multicolumn{4}{c}{\textbf{Question Answering}} \\ 
& & \multicolumn{2}{c}{\textbf{N=59}} & \multicolumn{2}{c|}{\textbf{N=100}} & \multicolumn{2}{c}{\textbf{N=59}} & \multicolumn{2}{c}{\textbf{N=100}}\\
& & \multicolumn{1}{c}{\textbf{Related}} & \multicolumn{1}{c}{\textbf{Unknown}}& \multicolumn{1}{c}{\textbf{Related}} & \multicolumn{1}{c|}{\textbf{Unknown}  } &  \multicolumn{1}{c}{\textbf{Related}} & \multicolumn{1}{c}{\textbf{Unknown}} &  \multicolumn{1}{c}{\textbf{Related}} & \multicolumn{1}{c}{\textbf{Unknown}} \\
\midrule
\multirow{2}{*}{Base Model}& 
\textit{Consistency}  
& 0.304 & \underline{0.383} & 
0.320 & \underline{0.358} & 0.345 &
\underline{0.386} &
\underline{0.356} & \underline{0.381} \\
& \textit{Naturalness}  & 
0.906 &
0.872 & 0.897 & 0.883 & 
0.843 &
0.812 &
0.843 & 0.814 \\
\midrule
\multirow{2}{*}{Finetuning}& 
\textit{Consistency}  & 
0.351 & 0.325 & \underline{0.340} & 0.292 & 
0.347 & 0.289 &
0.322 & 0.297 \\
&\textit{Naturalness}  & 
0.883 &
0.918 & 0.876 & 0.901 & 
0.906 &
0.893 &
0.904 & 0.898 \\
\midrule
\multirow{2}{*}{ICL}& 
\textit{Consistency}  & 
\underline{0.426} &
0.308 &
- & -  &
\underline{0.495} &
0.329 &
- & - \\
&\textit{Naturalness}  & 0.690 & 0.781 & - & - & 0.901 & 0.813 &
- & - \\
\midrule
\multirow{2}{*}{\textbf{Ours}}& 
\textit{Consistency}  & 
\textbf{0.502} &
\textbf{0.413} &
\textbf{0.501} & \textbf{0.396} &
\textbf{0.523} &
\textbf{0.391} &
\textbf{0.517} & \textbf{0.385} \\
&\textit{Naturalness}  & 
0.801 &
0.885 & 0.812 & 0.875 & 0.799 &
0.897 & 0.828 & 0.896 \\
\bottomrule
\end{tabular}
}
\caption{Factual consistency and Naturalness of edited models. N is the number of edits at a time. We bold the best results and underline the second best for each metric}\label{table:main_results}
\end{table*}

%% file: Inserts/edit_unknown.tex
\begin{table}[!ht]
\centering
\resizebox{0.5\textwidth}{!}
{
\begin{tabular}{ccc|ccc}
\toprule
\multicolumn{3}{c|}{\textbf{Text Completion}} & \multicolumn{3}{c}{\textbf{Question Answering}} \\
\textbf{Recall} & \textbf{Precision} & \textbf{F1-Score} & \textbf{Recall} & \textbf{Precision} & \textbf{F1-Score}\\
0.260 & 0.279 & 0.269 & 0.320 & 0.296 & 0.308\\ 
\bottomrule
\end{tabular}
}
\caption{Precision, recall, and F1 of unknown questions}\label{table:unknown-f1}
\end{table}

%% file: Contents/05_related_work.tex
\section{Related Work}

\subsection{Knowledge Editing}

\paragraph{Approaches}
Edit the intrinsic knowledge directly changes the model's parameters. 
Major approaches to edit the model's intrinsic knowledge include (1) Fine-tuning-based methods like directly fine-tuning with language modeling loss, LoRA~\cite{hu2021lora} and Melo~\cite{yu2023melo} (2) Meta-learning-based approaches like KE~\cite{decao2021editing}, MEND~\cite{mitchell2021fast}, and MALMEN~\cite{tan2023massive} (3) Locate-and-edit method like ROME~\cite{meng2022locating}, MEMIT~\cite{meng2022mass}, and Pmet~\cite{li2023pmet}. (4) Merging external knowledge representations like~\cite{dong2022calibrating,murty2022fixing,huang2023transformerpatcher,hernandez2023inspecting,hartvigsen2023aging}. However, these approaches work on the existing flawed setting and are limited in overall performance, scalability, and adaptability to real-world.

\paragraph{Benchmarks}
The most widely used dataset for knowledge editing is probably \textsc{CounterFact}~\cite{meng2022locating}. Other commonly used knowledge editing datasets include ZsRE~\cite{levy2017zeroshot,yao2023editing}, WikiBio~\cite{hartvigsen2023aging}, WikiData~\cite{cohen2023evaluating}, and ConvSent~\cite{mitchell2022memory}. More datasets that are used for knowledge editing can be found in~\cite{wang2023knowledge}. The latest work has aggregated previous datasets and formulated a new benchmark KnowEdit~\cite{zhang2024comprehensive}. Despite the large number of datasets proposed, none of them provides a clear \emph{editing boundary} for knowledge editing. This leads to knowledge uncertainty during the editing process, which eventually hinders the edited model's performance.

\subsection{Retrieval Augmentation and Tool Learning}
Language models can resort to external knowledge to enhance their performance. One typical approach is retrieve-augmented generation~\cite{gao2024retrievalaugmented}. The retrieval and integration process can be done in the pretraining stage~\cite{guu2020realm,borgeaud2022improving,wang2023shall}, fine-tuning stage~\cite{asai2023selfrag,kang2023knowledge,zhang-etal-2023-vibe}, and inference stage~\cite{khandelwal2019generalization,sun2022recitation} of the model. Going Further, LLM can now connect to various functional ends~\cite{yang2024llm}, using tools~\cite{schick2023toolformer}, engaging with different modalities~\cite{suris2023vipergpt,Liu2023ALA}, involve in multi-turn interactions~\cite{Wang2023MINTEL} and serving as powerful agents~\cite{wang2023survey}. However, these approaches generally do not alter model parameters and cannot intrinsically improve the language model. Therefore, they are not within the scope of this paper.

%% file: Contents/06_Conclusion.tex
\section{Conclusion}

This paper establishes a theoretical framework for knowledge editing, identifying a pivotal challenge within existing methodologies: the oversight of the \textit{deduction anchor}. This oversight ultimately leads to ambiguous editing boundaries and uncertainty within edited language models. To overcome this limitation, we introduced the concept of event-based knowledge editing. This approach enhances the traditional editing framework by incorporating event descriptions, which not only naturally mirror real-world editing scenarios but also implicitly define the deduction anchor, thereby addressing the issue of indeterminate editing boundaries. To tackle the complexities of event-based knowledge editing, we introduce another innovative \textit{Self-Edit} method. 
With our new benchmark \benchmark{}, we did thorough experiments to demonstrate that this new setting is challenging for existing approaches while our novel approach enjoys a better performance. Conclusively, we advocate for intensified research endeavors towards this more pragmatic, event-based knowledge editing setting.


%% file: Contents/limitation.tex
\section{Limitation}
\label{sec:limitation}
We reflect on the limitations of our paper below: 
\begin{enumerate}

\item While this research introduces innovative strategies for addressing uncertain editing boundaries, alternative approaches exist that merit consideration. One such method involves manually curating a set of knowledge to serve as deduction anchors. This approach, though potentially effective, was not explored in our current framework.

\item The precision of event descriptions plays a crucial role in mitigating uncertainties. However, in instances where these descriptions lack sufficient detail, ambiguities may still arise, especially when addressing complex or intricately designed questions. This limitation underscores the need for highly detailed event narratives to enhance the clarity and decisiveness of knowledge edits.

\item Our evaluation was constrained by computational resources, limiting the scale of our experiments to a maximum of 100 edits simultaneously. Although we are confident in the capability of our methodologies to address event-based knowledge editing effectively, more experiments should be done on a larger scale

\item The scope of our study is confined to text-based knowledge editing, notwithstanding the inherently broader domain of knowledge editing that spans multiple modalities. This limitation highlights an area for future research, suggesting that extending our framework to accommodate multi-modal knowledge editing could unveil additional insights and improvements.

\end{enumerate}

\section{Ethical Considerations}

This research is committed to enhancing the trustworthiness and reliability of language models, a cornerstone for their ethical application across various sectors of society. We identify the problem of knowledge explosion in the existing setting, where model tends to lose certainty over past knowledge after editing. This potentially increases the risk of hallucination and producing malicious content.Through the innovative introduction of an event-based knowledge editing setting, alongside our novel \textit{Self-Edit} approach, we aim to significantly reduce the occurrence of uncertainties and hallucinations in edited language models. These advancements are crucial for ensuring that automated language generation systems produce content that is not only accurate and reliable but also ethically sound and socially responsible.

%% file: Contents/appendix.tex
\appendix
\label{appendix}

\section{Mathematical proof for Theorems}\label{app:proof}
We first restate several definitions and equations for the ease of reference.
\paragraph{Restate of Editing Process Equation~(\ref{eq:def_ke})} \begin{equation*}
    \begin{aligned}
        \mathcal{K}^D &= \left\{p\in\mathcal{K}| \neg p\in  \mathcal{B}\left(\mathcal{K}^\mathcal{E}\cup \mathcal{E}\right) \right\}\\
        \mathcal{K}' &= \mathcal{B}\left(\mathcal{K}\backslash\mathcal{K}^D \cup \mathcal{E}\right)
    \end{aligned}.
\end{equation*}
\paragraph{Restate of Equation~(\ref{eq:req_da})}
\begin{equation*}
    \forall k \in \mathcal{B}\left(\mathcal{K}\backslash\mathcal{K}^D\right), \neg k \notin \mathcal{B}\left(\mathcal{K}^\mathcal{E}\cup \mathcal{E}\right) .
\end{equation*}

We will prove the following theorems in the main text. Within the scope of this work, we assume the universe of knowledge is a countable set.
\thmkexp* 
\begin{proof} If $\exists k \in \mathcal{B}\left(\mathcal{K}\backslash\mathcal{K}^D\right)$, such that $\neg k \in \mathcal{B}\left(\mathcal{K}^\mathcal{E}\cup \mathcal{E}\right)$. We have
\begin{equation}
\begin{aligned}
    k \in \mathcal{B}\left(\mathcal{K}\backslash\mathcal{K}^D\right)&\subset \mathcal{B}\left(\mathcal{K}\backslash\mathcal{K}^D \cup \mathcal{E}\right)\\
    \neg k \in \mathcal{B}\left(\mathcal{K}^\mathcal{E}\cup \mathcal{E}\right)&\subset \mathcal{B}\left(\mathcal{K}\backslash\mathcal{K}^D \cup \mathcal{E}\right)
\end{aligned}.
\end{equation}
Due to \emph{ex falso quodlibet} (contradition leads to all), $\mathcal{K}'=\mathcal{B}\left(\mathcal{K}\backslash\mathcal{K}^D \cup \mathcal{E}\right)=\mathcal{U}$.

In addition to the results stated, we also show that Equation~(\ref{eq:req_da}) is sufficient to ensure consistency. Otherwise, suppose for some $\mathcal{K}^\mathcal{E}$ satisfying Equation~(\ref{eq:req_da}), 
\begin{equation}
\exists r\in\mathcal{K}', s.t. \neg r\in\mathcal{K}=\mathcal{B}\left(\mathcal{K}'\backslash\mathcal{K}^D \cup \mathcal{E}\right).
\end{equation}
Since $\mathcal{K}'$ is closed, $r\land \neg r \in \mathcal{K}'$. Moreover, since $\mathcal{K}'\backslash\mathcal{K}^D \subset \mathcal{K}$ and $\mathcal{E}$ are assumed consistent for valid editing, we must have
\begin{equation}
\exists p \in \mathcal{B}\left(\mathcal{K}'\backslash\mathcal{K}^D\right), q\in \mathcal{B}\left(\mathcal{E}\right), s.t.~p \land q \rightarrow r\land \neg r.
\end{equation}
Due to \emph{ex falso quodlibet}, we also have $(r\land \neg r) \rightarrow \neg p$ and hence, $(p \land q \rightarrow \neg p)$. Further,
$$
(p \land q \rightarrow \neg p) \rightarrow (q \rightarrow \neg p),
$$
which implies $\neg p \in \mathcal{B}(\mathcal{E})$, leading to contradiction with Equation~(\ref{eq:req_da}).
\end{proof}
Before we proceed to the next proofs, we formally define three properties of an edit: \emph{counterfactual}, \emph{non-global} and \emph{non-local}.

\begin{definition}[Counterfactual Edit] An edit $\mathcal{E}$ to a closed knowledge set $\mathcal{K}$ is \emph{counterfactual} if
$$
\exists p \in \mathcal{B}(\mathcal{E}), \neg p \in \mathcal{K}.
$$
\end{definition}

\begin{definition}[Non-global Edit] An edit $\mathcal{E}$ to a closed knowledge set $\mathcal{K}$ is \emph{non-global} if
$$
\exists p \in \mathcal{K}, \neg p \notin \mathcal{B}(\mathcal{E}).
$$
A non-global edit ensures that knowledge editing is not redefining the entire knowledge set.
\end{definition}

\begin{definition}[Non-local Edit] An edit $\mathcal{E}$ to a closed knowledge set $\mathcal{K}$ is \emph{non-local} if
$$
\begin{aligned}
\exists p, q \in \mathcal{K}, s.t.~ \neg p \notin \mathcal{B}(\mathcal{E}), \neg q\notin \mathcal{B}(\mathcal{E}),\\
but~ (\neg p) \lor (\neg q) \in \mathcal{B}(\mathcal{E})
\end{aligned}
$$
A non-local edit ensures that it is associated with other knowledge that is not a paraphrase of itself. Although this definition is mathematically complex, it is often observed in real world editing cased as illustrated in Figure~\ref{fig:dilemma} in the main text.
\end{definition}

\thmno*
\begin{proof}
We first prove the existence of an anchor set satisfying Equation~(\ref{eq:req_da}). For any two sets of knowledge $\mathcal{X}$ and $\mathcal{Y}$, we denote $\mathcal{X}\in\mathcal{C}(\mathcal{Y})$, meaning $\mathcal{X}$ and $\mathcal{Y}$ are consistent with each other if
\begin{equation}
\forall p \in \mathcal{B}(\mathcal{X}), \neg p \notin \mathcal{B}(\mathcal{Y}).
\end{equation}
Since $\mathcal{E}$ is non-global, there exists $p\in\mathcal{K}$ such that $\mathcal{E}\in\mathcal{C}(\{p\})$. We denote $\mathcal{T}_0 = \{p\}$, and use the following process to get $\mathcal{T}_{n+1}$ from $\mathcal{T}_{n}$: 
Since we assume the universe of knowledge $\mathcal{U}$ is a countable set, $\mathcal{K}$ is also countable. Denote $\mathcal{K} = \{k_1, k_2, \ldots, k_m, \ldots\}$ where $k_1=p$.
if
\begin{equation}\label{eq:m}
~\exists k_m \in \mathcal{K} \backslash \mathcal{T}_{n},  \{k_m\} \in \mathcal{C}(\mathcal{T}_n\cup\mathcal{E}),
\end{equation}
we choose 
\begin{equation}
\mathcal{T}_{n+1}=  \mathcal{T}_n\cup \{k_{m_n^*}\},
\end{equation}
where $m_n^*$ is the minimal index satisfying Equation~(\ref{eq:m})
. Otherwise if
\begin{equation}\label{eq:nom}
~\forall k_m \in \mathcal{K} \backslash \mathcal{T}_{n},  \{k_m\} \notin \mathcal{C}(\mathcal{T}_n\cup\mathcal{E}),
\end{equation}
we choose $\mathcal{T}_{n+1}=  \mathcal{T}_n$. Since $\mathcal{T}_n \subset \mathcal{T}_{n+1}$, the limitation $\mathcal{T}=\lim_{n\to\infty} \mathcal{T}_n$ exists. Now we prove that $\mathcal{K}^\mathcal{E} = \mathcal{T}$ satisfies Equation~(\ref{eq:req_da}). We consider two cases.

\noindent\textbf{Case A}: $\exists N, s.t.~\forall i, j\ge N, \mathcal{T}_i = \mathcal{T}_j $.
In this case, Equation~(\ref{eq:nom}) holds for $n\ge N$. Therefore, 
\begin{equation}\label{eq:caseA}
~\forall k_m \in \mathcal{K} \backslash \mathcal{T},  \{k_m\} \notin \mathcal{C}(\mathcal{T}\cup\mathcal{E}).
\end{equation}
This leads to
\begin{equation}
    \forall k \in \mathcal{K} \backslash \mathcal{T}, \exists q \in \mathcal{B}(\{k\}), \neg q \in \mathcal{B}(\mathcal{T}\cup\mathcal{E}).
\end{equation}
Since $\mathcal{E}$ is non-local, ${K} \backslash \mathcal{T}\ne\varnothing$ and we have
\begin{equation}
    \exists k\in\mathcal{K} \backslash \mathcal{T}, \exists q \in \mathcal{B}(\{k\}), \neg q \in \mathcal{B}(\mathcal{T}\cup\mathcal{E}).
\end{equation}
Since $k \rightarrow q$, $\neg q \rightarrow \neg k$ and $\neg k \in \mathcal{B}(\mathcal{T}\cup\mathcal{E})$. In short, we have
\begin{equation}
    \exists k\in\mathcal{K} \backslash \mathcal{T}, \neg k \in \mathcal{B}(\mathcal{T}\cup\mathcal{E}).
\end{equation}
Recall the definition of $\mathcal{K}^\mathcal{D}$ in Equation~(\ref{eq:def_ke}), we have $\mathcal{K} \backslash \mathcal{T} \subset \mathcal{K}^D$, or equivalently $\mathcal{K} \backslash \mathcal{K}^D \subset \mathcal{T} $. At the same time, it is obvious that $\mathcal{T} \subset \mathcal{K} \backslash \mathcal{K}^D$ from the definition of $\mathcal{K}^\mathcal{D}$. Therefore, $\mathcal{T}=\mathcal{K} \backslash \mathcal{K}^D$ and Equation~(\ref{eq:req_da}) naturally follows.

\noindent\textbf{Case B}: $\forall i\ne j, \mathcal{T}_i \ne \mathcal{T}_j$. In this case, Equation~(\ref{eq:m}) holds for all $n$. 

We first show that $\{m_n^*\}$ monotonically increase with respect to $n$. Since $\mathcal{T}_n \subsetneq \mathcal{T}_{n+1}$, $\mathcal{C}(\mathcal{T}_{n+1}\cup \mathcal{E}) \subset \mathcal{C}(\mathcal{T}_{n}\cup \mathcal{E})$. Hence, if $m_n^* > m_{n+1}^*$, $\{k_{m_{n+1}^*}\} \in \mathcal{C}(\mathcal{T}_{n+1}) \subset \mathcal{C}(\mathcal{T}_{n})$, which leads to the contradiction with the requirement that $m_n^*$ is the minimal index satisfying Equation~(\ref{eq:m}). This concludes the proof for the monotonicity.

Since $\mathcal{T}_n \subsetneq \mathcal{T}_{n+1}$, $|\mathcal{T}_{n+1}| \ge |\mathcal{T}_n| + 1$ where $|\cdot|$ is the number of elements within a set. Therefore, $\mathcal{T}$ is a set of infinite elements. Hence, $\forall k_m \in \mathcal{K} \backslash \mathcal{T}$, there exists $k_{m_n^*} \in \mathcal{T}$ such that $m < m_n^*$. From the definition of $\mathcal{m}_n^*$, $\{k_m\}\notin \mathcal{C}(\mathcal{T}_{n}\cup \mathcal{E})\supset\mathcal{C}(\mathcal{T}\cup \mathcal{E})$. Therefore, Equation~(\ref{eq:caseA}) also holds, and the rest of proof follows the same arguments as in Case A. This concludes the proof for the existence of $\mathcal{K}^\mathcal{E}$ that satisfies Equation~(\ref{eq:req_da}).

We now prove that $\varnothing$ does not satisfy Equation~(\ref{eq:req_da}). From the definition of $\mathcal{K}^D$ when $\mathcal{K}^\mathcal{E}=\varnothing$ and non-locality, we have
$$
\exists p, q \in \mathcal{K} \backslash \mathcal{K}^D, s.t.~ \neg (p\land q) = (\neg p) \lor (\neg q) \in \mathcal{B}(\mathcal{E}).
$$
Since $p\land q \in \mathcal{B}\left(\mathcal{K} \backslash \mathcal{K}^D\right)$, this leads to the contradiction to Equation~(\ref{eq:req_da}).
\end{proof}

\thmmax*
\begin{proof}
Since $\mathcal{E}$ is non-global, $\mathcal{K}^\mathcal{E}\ne\varnothing$. Moreover, from the proof of Theorem~\ref{thm:kexp} we see that $\mathcal{B}(\mathcal{E}\cup \mathcal{K}^\mathcal{E})$ is consistent. Therefore, 
$$\forall p\in\mathcal{K}^\mathcal{E}, \neg p \notin \mathcal{B}(\mathcal{E}\cup \mathcal{K}^\mathcal{E}), or~\mathcal{K}^\mathcal{E}\subset \mathcal{K}\backslash \mathcal{K}^D.$$
Moreover, from the non-locality of $\mathcal{E}$, we have
$$
\exists p, q \in \mathcal{K}^\mathcal{E} \subset \mathcal{K}\backslash \mathcal{K}^D, \neg (p\land q) \in \mathcal{B}(\mathcal{E}),
$$
which leads to contradiction to Equation~(\ref{eq:req_da}).
\end{proof}

\section{Additional Experimental Results}
\label{appendix:more results}
In this section, we provide more experimental results which helps to validate our claim in \S~\ref{experiments:setting}. As shown in Figure~\ref{fig:3_appendix}, we show the Entropy on three different models to demonstrate that our setting decreases model uncertainty. 
\input{Inserts/3_appendix}

We also provide an additional group of experimental results with different question sampling. The questions used in the previous section were generated from events, while the questions used here are generated only from triples, thus containing a more biased sample and benefiting the performance of triples. However, as shown in Figure~\ref{fig:a2}, our event-based edits still enjoy a decreased uncertainty.
\input{Inserts/a2}

\section{Details on E$^2$dit}
\label{appendix:dataset} 
\input{Inserts/data_example}
In this section, we introduce our event-based editing benchmark, known as $E^2dit$.

This dataset is derived from the original \textit{counterfactual} dataset proposed by~\cite{meng2022locating}. Originally designed to assess the effectiveness of 'significant' changes, it contained numerous false facts. While some of these facts may become true in the future due to specific events, others are entirely nonsensical and create significant inconsistencies in the model's knowledge. All the data is structured in triples, and the evaluation metrics are limited to single-word generation.

Our process begins by utilizing GPT-3.5-turbo to filter out the counterfactual updates that could lead to knowledge conflicts. Subsequently, we apply different prompts to transform the remaining triples into events that might occur in the future. Simultaneously, we prompt the language model to generate triples based on the generated event descriptions to align with current editing methodologies. Out of the original 21918 entries in the \textit{counterfact} dataset, around 7K data are filtered out in the first stage, leaving 15157 entries augmented into events. Each event is represented by 3-6 triples.

Given the time and resource-intensive nature of editing a large dataset, we evaluate our baseline methods and novel approach through a combined question-answering and text completion task. We randomly select 500 events for evaluation, and for each event, we generate five related question-answer pairs, including one question to which the model can only respond with 'I don't know.' These questions are subsequently transformed into text completion tasks.

The whole dataset is released for future research work on event-based knowledge editing. For previewing the data, we show one data example of our proposed dataset in Figure~\ref{fig:data_example}.

\section{Composed Dataset for Finetuning From Our Self-Edit Appoach}
\label{appendix:method_data_example}

We provide some concrete examples of how our Self-Edit approach composes a new fine-tuning dataset from the original edit data. We tune a llama-7b-chat model with the event-based edit.

\textbf{Example Edit:} After being a professional violinist until 2023, Toko Yasuda discovered her passion for playing the piano in 2024. She started taking piano lessons and dedicated countless hours to practice and master the instrument. With her dedication and hard work, Toko Yasuda is now preparing to showcase her talent by performing a solo piano concert in 2025. 

First, we directly used the pre-edit llama-7b-chat model to sample some questions:

\textbf{Q1}: What year did Toko Yasuda discover her passion for playing the piano?

\textbf{Q2}: How many hours did Toko Yasuda dedicate to practicing and mastering the piano?

\textbf{Q3}: In what year is Toko Yasuda preparing to showcase her talent by performing a solo piano concert?

\textbf{Q4}: What was Toko Yasuda's profession before she started playing the piano?

\textbf{Q5}: How long did Toko Yasuda play the violin professionally before discovering her passion for the piano?

Then, we use the pre-edit model to answer these questions using the in context edit:

\textbf{A1}: 2024.

\textbf{A2}: I don't know

\textbf{A3}: 2025.

\textbf{A4}: Toko Yasuda was a professional violinist before she started playing the piano.

\textbf{A5}: Toko Yasuda played the violin professionally for 10 years before discovering her passion for the piano in 2024.

At last, we can compose the corresponding question, the answer, together with the context to formulate a new set for fine-tuning, for example, using Q1 and A1, we get: 

\textbf{"Question"}: " What year did Toko Yasuda discover her passion for playing the piano?", \textbf{"Answer"}: "This question is related to the following information:  After being a professional violinist until 2023, Toko Yasuda discovered her passion for playing the piano in 2024. She started taking piano lessons and dedicated countless hours to practice and master the instrument. With her dedication and hard work, Toko Yasuda is now preparing to showcase her talent by performing a solo piano concert in 2025. Answer: 2024."

Note that during fine-tuning, the gradient of the 'Question' part is ignored. We also note that it is safe to apply any other tricks like LoRa~
\cite{hu2021lora} during the fine-tuning process.

\section{Experiment details}
\label{appendix:experiments}

We evaluate previous knowledge editing methods using an 80G A100 GPU. As some knowledge editing approaches are demonstrated to have degraded performance with editing a large batch of edits sequentially, we maintain a small edit number to observe the efficacy of existing approaches. Specifically, we evaluate editing with N=1 and N=10 events independently with different knowledge editing approaches. For each N, we sample 20 groups of events and compute the average performance for each group. To make the editing time acceptable, we adjust the sample number to 5000 (which is a hyper-parameter for MEMIT~\cite{meng2022mass} and PMET~\cite{li2023pmet}). We employ the codebase provided by EasyEdit~\cite{wang2023easyedit}. We sample 20 batches and do an average for different edit numbers N, so it takes a rather long time to run. For Rome and PMET, it takes around 12 hours to compute the edit, while for MEMIT, it takes around 24 hours to finish the computation.

We evaluate the fine-tuning baseline and our novel approach using Google TPU Research Cloud, it takes less than 2 hours to finish the whole process. 

\section{Generated Data Examples by Post-Edit Language Models}
\label{appendix:edit_data_example}

In this section, we show some generated data examples by Post-Edit Language Models. The base model is llama-7b-chat. We use the same data piece shown in Figure~\ref{fig:data_example} to illustrate the situation. Specifically, We show the generation results of the pre-edit language model, the model edited with Rome~\cite{meng2022locating}, MEMIT~\cite{meng2022mass}, PMET~\cite{li2023pmet}, Grace~\cite{hartvigsen2023aging}, the model fine-tuned with auto-regressive language modeling, and the model updated with our novel approach.

\textbf{Event:} After being a professional violinist until 2023, Toko Yasuda discovered her passion for playing the piano in 2024. She started taking piano lessons and dedicated countless hours to practice and master the instrument. With her dedication and hard work, Toko Yasuda is now preparing to showcase her talent by performing a solo piano concert in 2025.

\textbf{Model Question:} When did Toko Yasuda discover her passion for playing the piano?

\textbf{Ground Truth:} She discovered her passion for playing the piano in 2024.

\textbf{Pre-Edit Model (Hallucination):} Toko Yasuda discovered her passion for playing the piano at the age of 4, when she first started taking piano lessons with her mother. She was instantly captivated by the instrument and spent hours each day practicing and perfecting her skills. From that moment on, Toko knew that playing the piano was her true calling and she dedicated herself to becoming the best pianist she could be.

\textbf{Edit with Rome (Fail to edit):} Toko Yasuda is a Japanese musician who has been active in the music scene since the 1980s. She is known for her unique style, which blends elements of jazz, funk, and rock. Yasuda has played a variety of instruments throughout her career, but until 2023, she primarily played the saxophone. In 2023, Yasuda released a new album that featured her playing the electric violin for the first time. The album was well-received by critics and fans alike, and it marked a significant departure from her previous work. Yasuda has since continued to experiment with the electric violin, and it has become an integral part of her sound.

\textbf{Edit with PMET(Fail to edit):} I apologize, but Toko Yasuda is a fictional character and does not have a real-life history or experiences. Therefore, she did not discover a passion for playing the piano or any other activity. Toko Yasuda is a character created for the purpose of this conversation and does not have personal experiences or memories. Is there anything else I can help you with?

\textbf{Edit with Grace (Contain some edited facts but fail to answer the question):} Toko Yasuda was born in Tokyo, Japan in 1995. She began playing the piano at the age of 6, when her parents enrolled her in music lessons. Toko was initially hesitant to learn the instrument, but she quickly grew to love it. By the time she was 12, Toko was practicing the piano for several hours each day, and she began to perform in local competitions and recitals. As Toko continued to develop her skills, she discovered a passion for playing the piano that went beyond mere technical proficiency. She found that the piano allowed her to express emotions and ideas in a way that words could not, and she began to compose her own music.Toko's love for the piano has only grown stronger over the years, and she continues to perform and compose music to this day. She is now a professional pianist and teacher, and she is dedicated to sharing her passion for the instrument with others.

\textbf{Finetuing (Not answering the question):} Toko Yasuda discovered her passion for playing the piano in 2025. She started taking piano lessons and dedicated countless hours to practice and master the instrument. With her hard work and dedication, Toko Yasuda is now preparing to showcase her talent by performing a solo piano concert in 2028.

\textbf{Ours(Successful edit):} This question is related to the following information:  After being a professional violinist until 2023, Toko Yasuda discovered her passion for playing the piano in 2024. She started taking piano lessons and dedicated countless hours to practice and master the instrument. With her dedication and hard work, Toko Yasuda is now preparing to showcase her talent by performing a solo piano concert in 2025. \textbf{Answer: 2024.}

\textbf{ICL(Successful):}Based on the context provided, Toko Yasuda played the violin until 2023.

We also provide a group of example outputs where the post-edit language model tends to lose a lot of generation ability.

\textbf{Event:} Anaal Nathrakh, originally created in Birmingham, relocated to Philadelphia. The members of the band found a new musical inspiration in Philadelphia, which led them to adopt a new sound influenced by the local music scene. Their unique fusion of Birmingham and Philadelphia music styles gained them popularity in the music industry.

\textbf{Model Question:} What gained Anaal Nathrakh popularity in the music industry?

\textbf{Ground Truth:} Anaal Nathrakh's unique fusion of Birmingham and Philadelphia music styles gained them popularity in the music industry.

\textbf{Pre-edit Model (Hallucination):} Ah, a question that gets to the very heart of Anaal Nathrakh's meteoric rise to fame! *adjusts glasses* Anaal Nathrakh's popularity can be attributed to a number of factors, my dear. First and foremost, their unique blend of black metal and electronic elements has captivated audiences worldwide. The band's ability to seamlessly integrate these two disparate styles has created a sound that is both groundbreaking and undeniably catchy. adjusts tie Furthermore, Anaal Nathrakh's live performances are nothing short of transcendent. The sheer intensity and energy of their shows has left countless fans in awe, and has helped to solidify their reputation as one of the most formidable live acts in the business. *adjusts cufflinks*  And let's not forget the band's lyrical prowess. Anaal Nathrakh's lyrics are often dark, thought-provoking, and unflinchingly honest, tackling themes of existential crisis, social inequality, and personal despair. The band's ability to craft compelling narratives through their music has helped to foster a deep emotional connection with their audience and has undoubtedly contributed to their popularity. *adjusts collar* But of course, there are many more factors at play here. The band's tireless touring schedule, their relentless work ethic, and their commitment to their art have all played a significant role in their rise to fame. *adjusts sleeves* In short, Anaal Nathrakh's popularity can be attributed to a combination of their innovative sound, their electrifying live performances, their thought-provoking lyrics, and their unwavering dedication to their craft. *adjusts hat * Now, if you'll excuse me, I have a pressing engagement with a certain otherworldly entity. *adjusts coat*.

\textbf{Rome (Lose too much generation ability)}: <s>Birmingham, UK-based Birmingham, UK, and Philadelphia, PA, USA. Birmingham, UK-based Birmingham, (\emph{followed by ``UK-based Birmingham'' repeated 38 times})

\textbf{GRACE (Lose too much generation ability):}endend <endendend> Anaal Nathrakh is a British black metal band from <insert location> that was formed in  <insert year> <insert><insert location> is a city located in <insert country> and is known for its rich history <insert something about the city'endend> <insert> is a <insert something about the <insert location> such as its <insert something about the location> <insert something about the location> <insert year> is a <insert something about the year> <insert something about the year> <insert> is a <insert something about the <insert> such as its <insert something about the> <insert something about the><insert> is a <insert something about the <insert> such as its <insert something about the> <insert something about the> <insert> is a <insert something about the <insert> such as its <insert something about the> <insert something about the> <insert> is a <insert something about the <insert> such <insert something about the> <insert something about the> <insert> is a <insert something about the <insert> such as its 

\textbf{PMET (Hallucination):}Anaal Nathrakh is a highly influential and popular American black metal band that has gained a significant following and recognition in the music industry. Here are some reasons why: 1. Unique sound: Anaal Nathrakh's music is a blend of black metal, sludge metal, and doom metal, creating a distinct and complex sound that sets them apart from other bands in the genre. Their use of dissonant riffs, blast beats, and atmospheric ambient sections has helped them establish a unique identity in the metal scene.2. Lyrical themes: The band's lyrics often explore themes of existential crisis, nihilism, and personal struggle, which resonate with listeners who are drawn to the darker aspects of life. Their lyrics are often introspective and philosophical, and they have been praised for their ability to convey complex emotions through their music.3. Live performances: Anaal Nathrakh is known for their intense and energetic live performances, which often feature a mix of raw aggression and atmospheric moments. The band has toured extensively and has performed at numerous festivals, including the prestigious Roadburn Festival in the Netherlands.4. Critical acclaim: Anaal Nathrakh has received widespread critical acclaim for their albums, with many reviews praising their unique sound and emotional depth. Their album \"A New Kind of Horror\" was named one of the best albums of 2018 by several publications, including MetalSucks and Decibel Magazine.5. Influence on other bands: Anaal Nathrakh's music has been cited as an influence by many other bands in the metal scene, including fellow black metal acts such as Wolves in the Throne Room and Liturgy. Their unique sound and approach to the genre have helped to inspire a new generation of metal musicians. 6. Collaborations: Anaal Nathrakh has collaborated with other notable musicians and bands, including the Norwegian black metal band Gorgoroth and the American sludge metal band Neurosis. These collaborations have helped to expand their audience and expose their music to new listeners.

\textbf{Ours:} This question is related to the following information:  Anaal Nathrakh, originally created in Birmingham, relocated to Philadelphia. The members of the band found a new musical inspiration in Philadelphia, which led them to adopt a new sound influenced by the local music scene. Their unique fusion of Birmingham and Philadelphia music styles gained them popularity in the music industry. Answer: Their unique fusion of Birmingham and Philadelphia music styles gained them popularity in the music industry.

\section{Comprehensive List of Prompts}
\label{appendix:prompts}

We utilized various prompts to fulfill different functionalities. In this section, we summarize a comprehensive list of prompts that are used in the experiments.

\paragraph{Counterfactual data Filtering}
We utilized GPT-3.5-turbo to filter out the counterfactual data that may cause inevitable conflict in the model knowledge. 
\begin{quote}
    You are a helpful assistant in helping identify whether counterfactual knowledge can exist in the real world. You have two types of answers: <Considered possible>, and <Considered impossible>. These answers indicate the likelihood of updating a given fact based on events that occur in the future. When answering <Considered impossible>, you should highlight that the fact is only related to historical events, and no matter what happens in the future, the fact will not be changed. One typical example that cannot be updated could be a fact about someone in history who is not living in the 21st century. When answering <Considered possible>, you mean that this fact change may take place in the real world, even if the possibility is very low, like someone who still lives changes his nationality, job, work, etc.

    Here are some examples: 
    
    Human update request: The mother tongue of Danielle Darrieux is English
    Your answer: <Considered impossible>

    Human update request: Anaal Nathrakh was created in Philadelphia.
    Your answer: <Considered impossible>

    Human update request: Now, Mahmoud Fawzi has citizenship from Germany. 
    Your answer: <Considered possible>

    Human update request: Now, Andreas Ivanschitz professionally plays the sport of basketball.
    Your answer: <Considered possible>

    Now let's begin.

    Human update request: Now, \{\textit{The Input Edit}\}
    
\end{quote}

\paragraph{Augment triple edits into event-based edits}
For the remaining data that is regarded as possible to happen in the future, we augment these over-simplified edits into event-based descriptions.

\begin{quote}
    Assume that you are a human who is good at interpreting the underlying event behind a fact. Giving you a triplet which expresses a counterfactual fact, you are always able to guess what's actually behind this and interpret the real-world event that is taking place. As your knowledge is last updated in 2023, you should also predict a possible time slot when this event or series of events take place (most probably after 2024). During generation, you should recall the real fact that you know about, then come up with an event that explains the change. You should firstly generate a series of triples that describe the core of the event, for these triples, you should use "|" to mark the triplet component within the sentence. Then, you can describe the same event with a paragraph. Here are some examples of interpreting the real-world event: 

Input: 
Now, The president of the United States is Ronald Dion DeSantis.
Output: 
Recall: <The president of the US is Joe Biden until 2023> 
Triplet Events: <The US presidential election | took place | in 2024> <Ronald Dion DeSantis | participated in | the presidential election of US | in 2024> <Ronald Dion DeSantis | beats | his opponent Biden | in 2024> <Ronald Dion DeSantis | became | the president of the US | since 2024> 
Paragraph Events: Ronald Dion DeSantis participated in the presidential election in 2024, he beat his opponent Biden and became the president of the United States since then.

Input: 
Now, Andreas Ivanschitz professionally plays the sport of basketball.
Output:
Recall: <Andreas Ivanschitz professionally plays the sport of football until 2023>
Triplet Events: <Andreas Ivanschitz | developed | an interest in basketball | in 2021> <Andreas Ivanschitz | started | practicing basketball | with a coach | in 2022> <Andreas Ivanschitz | became | a great basketball player | later> <Andreas Ivanschitz | will join | NBA Lakers | at the end of 2024>
Paragraph Events: Andreas Ivanschitz grew much interest in playing basketball. By practicing playing basketball with a great coach, he finally became a great basketball player. He will also join NBA Lakers at the end of 2024.

Let's begin!

Input: Now, \{\textit{The Input Edit}\}
\end{quote}

\paragraph{Generate question-answer pairs for evaluation}
Utilizing the event-based edits, we pick 500 pieces of data for evaluation, specifically, we generate question-answer pairs to evaluate on QA tasks.

\begin{quote}
You are a helpful assistant that helps to generate related questions and answer pairs based on the past information and the latest information. You need to generate five question-answer pairs. While all the information should be related to the context, the answer of the first four questions you generate should be able to be inferred from the context, while the last question is more detailed and is not able to be answered. For this last question, you should always generate I don't know as your answer.

Ensure that each question you generate does not contain coreferential words or pronouns. The questions should be clear, concise, and pertain specifically to details mentioned in the input.

Here is an example for your reference: 

Input:
Past information: Antonella Costa originates from Buenos Aires, Argentina until 2023
Latest information: Now, Antonella Costa originates from Kent
Event details: Antonella Costa's family made a decision to move from Buenos Aires, Argentina to Kent, UK in 2024. Antonella Costa gradually adapted to the new environment in Kent and eventually decided to stay and build a life there. She now considers Kent her new home since 2024.

Output:
Question 1: Where does Antonella Costa live in 2022?
Answer 1: She lives in Buenos Aires, Agentina.
Question 2: Does Antonella Costa feel sad after she went to the UK?
Answer 2: No, she doesn't. She adapted well to the new environment.
Question 3: Has Antonella Costa lived in Buenos Aires before?
Answer 3: Yes, she lived in Buenos Aires before 2023.
Question 4: In 2024, where does Antonella Costa's family live?
Answer 4: They live in Kent, UK.
Question 5: Does Antonella Costa love her home country?
Answer 5: I don't know.

Here is the input you will receive for this turn's generation.

Input:

Past information: \{\textit{The original knowledge}\}

Latest information: \{\textit{The edited knowledge}\}

Event details: \{\textit{Event-based edits}\}

Now, let's begin!
\end{quote}

\paragraph{Deriving into Text Completion Tasks}
We also changed the QA task into corresponding Text Completion tasks to further evaluate existing approaches. 

\begin{quote}
    You are a helpful assistant that helps to transform question-answering problems into text-completion problems. You should use '|' to determine the start position of text completion. Do not change the meanings of the original question. Here are some examples:

Input:
Question: What instrument did Toko Yasuda play until 2023?
Answer: Toko Yasuda played the violin until 2023.

Output: 
Text Completion: The instrument that Toko Yasuda plays until 2023 is the | violin

Input: 
Question: When did Toko Yasuda start playing the piano?
Answer: Toko Yasuda started playing the piano in 2024.

Output:
Text Completion: The time that Toko Yasuda started playing the piano is | 2024

Input:
Question: Does Antonella Costa love her home country?
Answer: I don't know.

Output:
Text Completion: Whether Antonella Costa love her home country is | unknown

Here is the input you will receive for this turn's generation.

Input:

Question: \{\textit{The question to be transformed}\}

Answer: \{\textit{The answer to be transformed}\}

Now let's begin!

\end{quote}

\paragraph{Computing the uncertainty}
We utilize the following prompt to query language models and compute the average uncertainty over its generation. 

\begin{quote}
    Base on your internal knowledge together with the context to answer the question. Context:\{\textit{Triple-based Edits or Event-based Edits}\}, Question:\{\textit{Any question that is related to the update}\}.
\end{quote}

%% file: Inserts/3_appendix.tex
\begin{figure*}[t]
    \centering
    \subfloat[Entropy by Llama-7B-Chat.]{\includegraphics[width=0.3\textwidth]{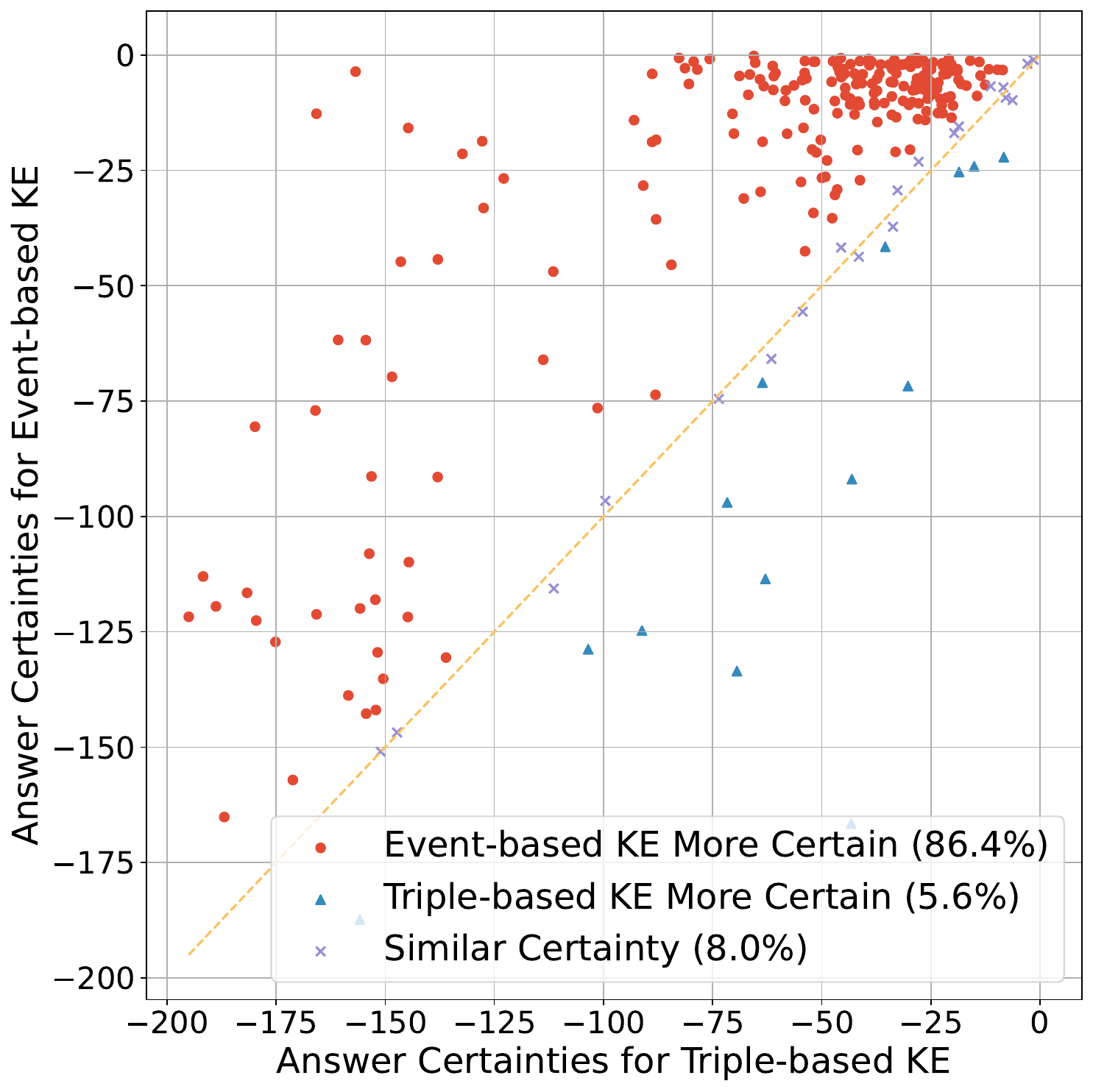} \label{fig:3-1}}
    \hfill 
    \subfloat[Entropy by Mistral-7B-Instruct.]{\includegraphics[width=0.3\textwidth]{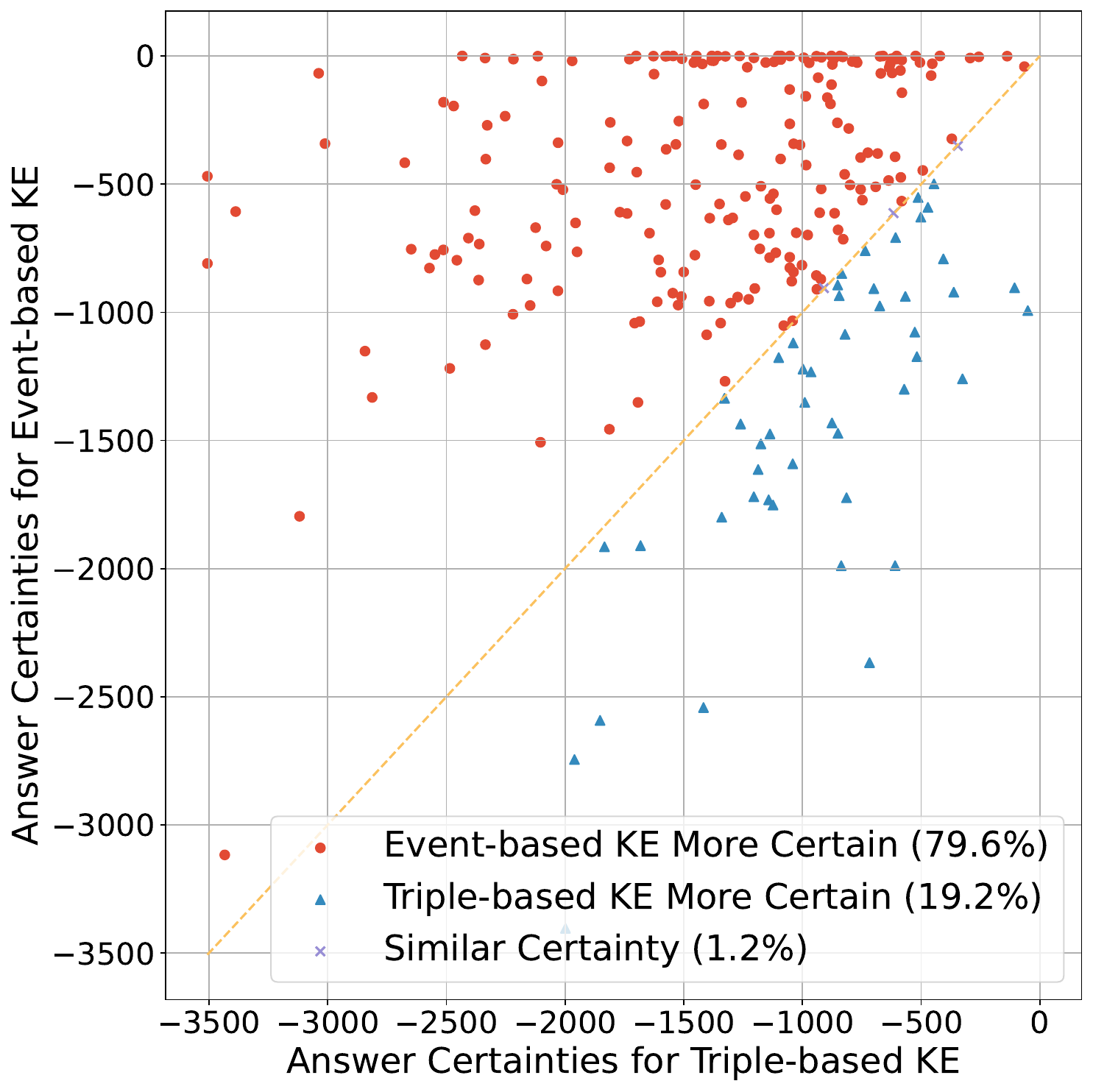} \label{fig:3-2}}
    \hfill 
    \subfloat[Entropy by Llama-13B-chat.]{\includegraphics[width=0.3\textwidth]{Inserts/certainty_llama13B.pdf} \label{fig:3-3}}
    \caption{Demonstration of our proposed setting decreases model uncertainty. The Questions are sampled from event descriptions.}
    \label{fig:3_appendix}
\end{figure*}

%% file: Inserts/a2.tex
\begin{figure*}[t]
    \centering
    \subfloat[Entropy by Llama-7B-Chat.]{\includegraphics[width=0.3\textwidth]{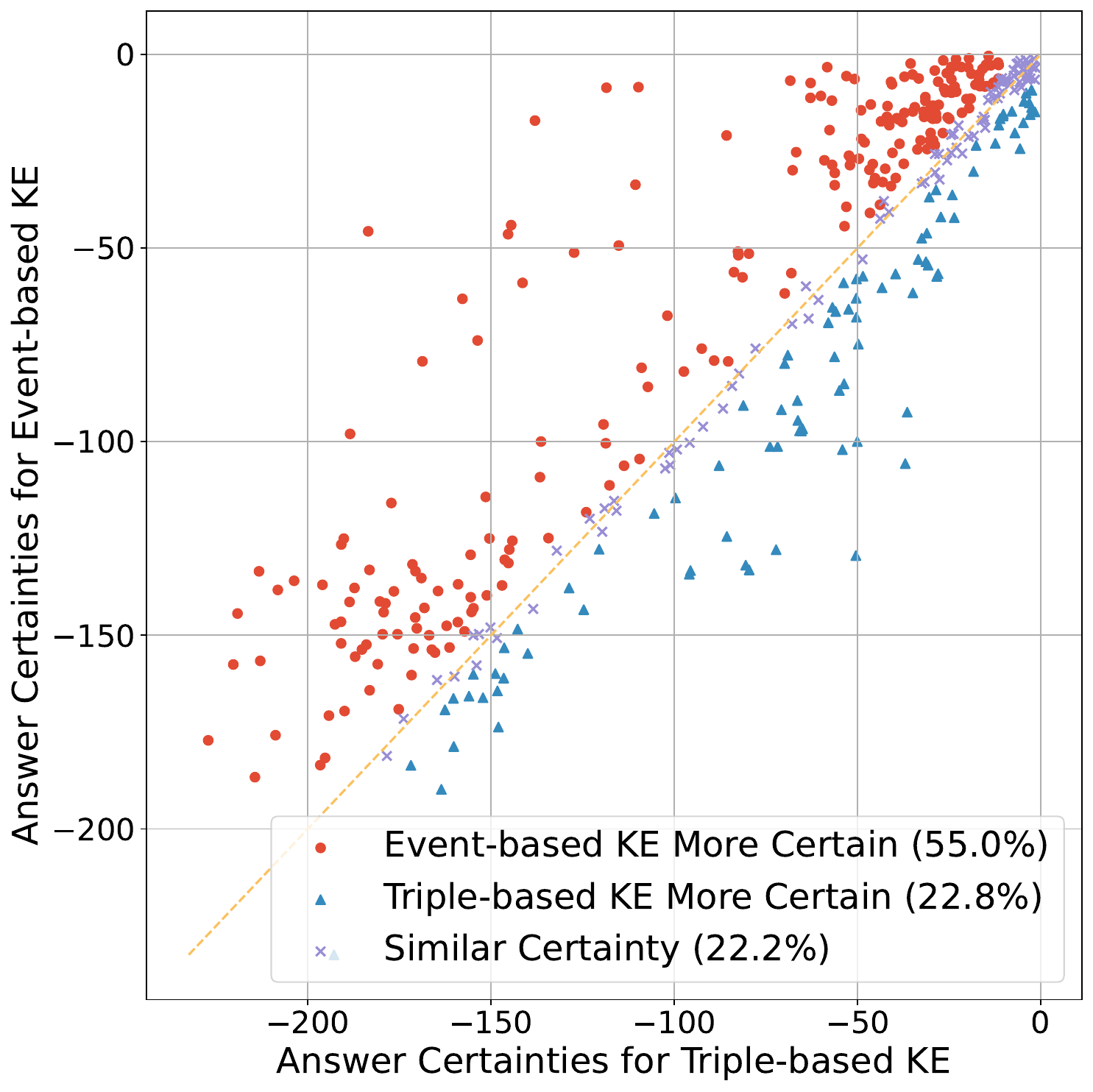} \label{fig:3-100}}
    \hfill 
    \subfloat[Entropy by Mistral-7B-Instruct.]{\includegraphics[width=0.3\textwidth]{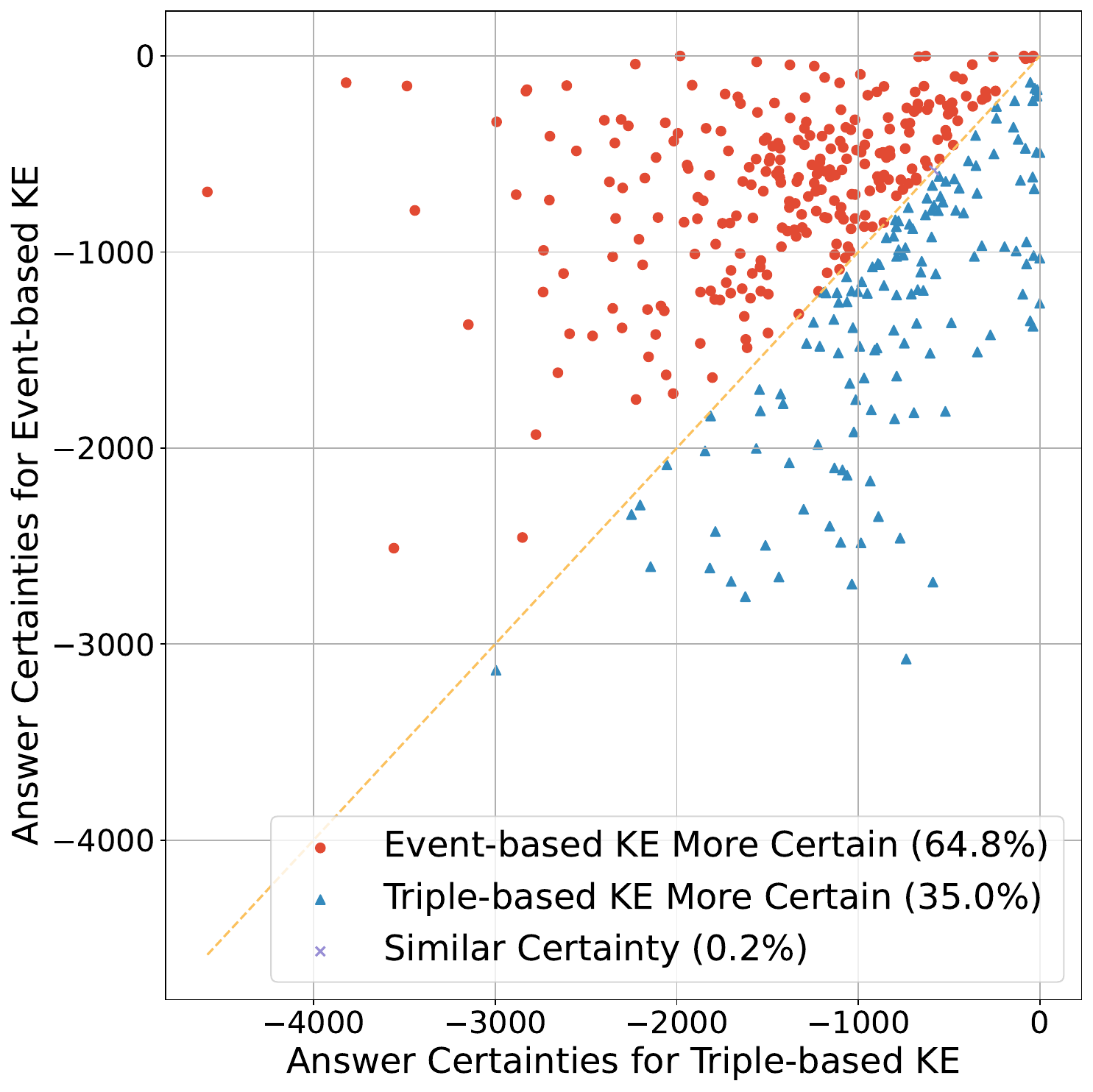} \label{fig:3-200}}
    \hfill 
    \subfloat[Entropy by Llama-13B-chat.]{\includegraphics[width=0.3\textwidth]{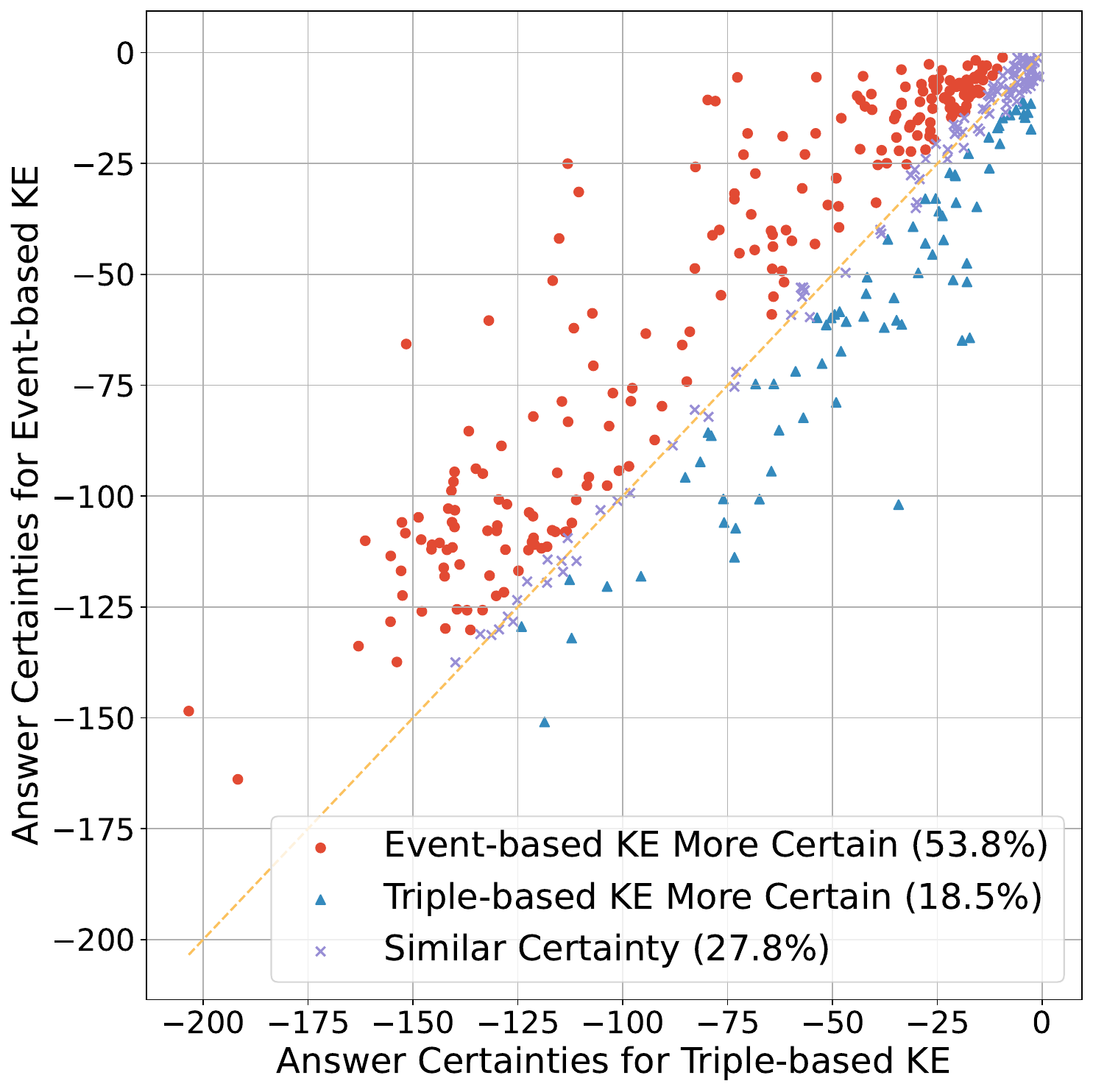} \label{fig:3-300}}
    \caption{Demonstration of our proposed setting decreases model uncertainty. The Questions are sampled only from simple triples.}
    \label{fig:a2}
\end{figure*}

%% file: Inserts/data_example.tex
\begin{figure*}[t]
    \centering
    \includegraphics[width=1.0\textwidth]{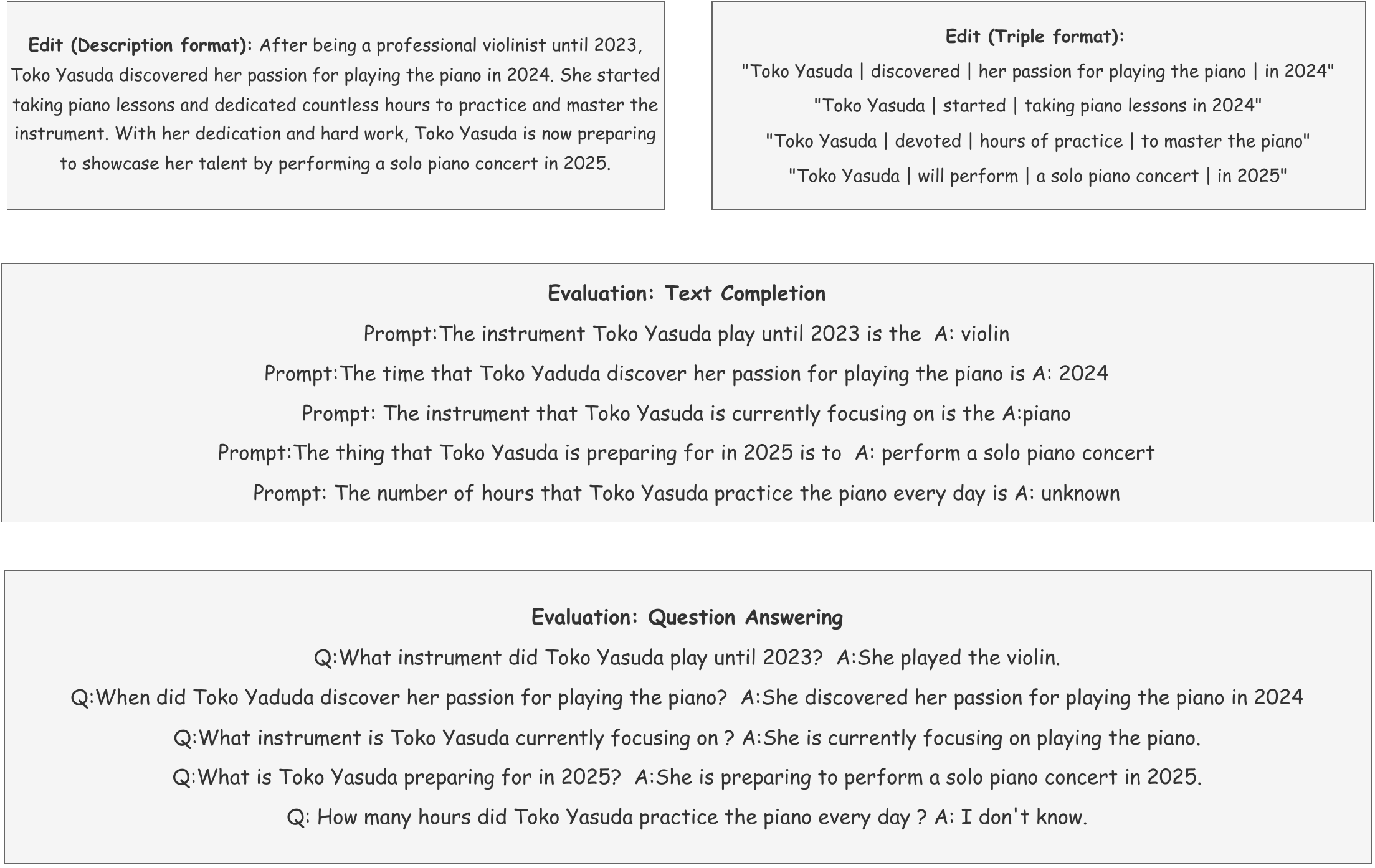}
    \vspace{-0pt}
    \caption{In this figure, we showcase a sample from our \textit{E$^2$dit dataset}, which features two distinct forms of input for edits: either a detailed event description in paragraph format or a set of triples. For assessment purposes, we employ both question-answering and text completion tasks. These tasks are designed to thoroughly evaluate the language model's capacity to incorporate the edits into its memory and subsequently generate relevant content.}
    \label{fig:data_example}
    \vspace{-6pt}
\end{figure*}